\documentclass{article}

\PassOptionsToPackage{numbers}{natbib}
    \usepackage[final]{neurips_2022}

\usepackage[utf8]{inputenc} %
\usepackage[T1]{fontenc}    %
\usepackage{hyperref}       %
\usepackage{url}            %
\usepackage{booktabs}       %
\usepackage{amsfonts}       %
\usepackage{nicefrac}       %
\usepackage{microtype}      %
\usepackage{xcolor}         %
\usepackage{adjustbox}
\usepackage{multirow}
\RequirePackage{algorithm}
\RequirePackage{algorithmic}
\usepackage{wrapfig}

\usepackage{graphicx}
\usepackage{subcaption}
\usepackage{breqn}
\usepackage{tabularx}

\usepackage{hyperref}

\usepackage{amsmath}
\usepackage{amssymb}
\usepackage{mathtools}
\usepackage{amsthm}
\usepackage{thmtools} 
\usepackage[font=small,labelfont=bf]{caption}
\usepackage{makecell}
\usepackage{wrapfig}

\usepackage[capitalize,noabbrev]{cleveref}

\theoremstyle{plain}
\newtheorem{theorem}{Theorem}[section]

\newtheorem{lemma}{Lemma}[section]

\crefname{corollary}{corollary}{corollaries}
\theoremstyle{definition}
\newtheorem{definition}{Definition}[section]
\crefname{definition}{definition}{definitions}
\newtheorem{assumption}{Assumption}[section]
\crefname{assumption}{assumption}{assumptions}
\newtheorem{example}{Example}[section]
\crefname{example}{example}{examples}
\theoremstyle{remark}
\newtheorem{remark}{Remark}[section]
\crefname{remark}{remark}{remarks}

\usepackage[textsize=tiny]{todonotes}

\usepackage{shortcuts}

\newcommand{\tauk}{\tau{(k)}}
\newcommand{\taui}{\tau_i}
\newcommand{\tauki}{\tau_i{(k)}}

\newcommand{\tauhz}{\hat{\tau}{(0)}}
\newcommand{\tauhk}{\hat{\tau}{(k)}}
\newcommand{\tauhzi}{\hat{\tau}_i{(0)}}
\newcommand{\tauhki}{\hat{\tau}_i{(k)}}

\newcommand{\sigmahksi}{\hat{\sigma}_{i}^{2}(k)}

\newcommand{\sigmahki}{\hat{\sigma}_{i}(k)}
\newcommand{\sigmaks}{\sigma^{2}(k)}
\newcommand{\sigmaksi}{\sigma_{i}^{2}(k)}

\newcommand{\sigmahzsi}{\hat{\sigma}_{i}^{2}(0)}

\newcommand{\sigmazsi}{\sigma_{i}^{2}(0)}

\newcommand{\muki}{\mu_i(k)}
\newcommand{\Lki}{\hat{L}_i{(k)}}
\newcommand{\Uki}{\hat{U}_i{(k)}}
\newcommand{\Thnki}{\hat{T}_{N}(k, i)}
\usepackage{csquotes}
\usepackage[inline]{enumitem} 

\title{Falsification before Extrapolation \\ in Causal Effect Estimation}

\author{%
  Zeshan Hussain\thanks{Equal contribution. Order determined alphabetically. For code and instructions on data access, please visit: \url{https://github.com/clinicalml/rct-obs-extrapolation}} \\
  MIT CSAIL \& IMES\\
  Cambridge, MA \\
  \texttt{zeshanmh@mit.edu} \\
   \And
   Michael Oberst${^*}$    \\
   MIT CSAIL \& IMES \\
   Cambridge, MA \\
   \texttt{moberst@mit.edu} \\
   \AND
   Ming-Chieh Shih${^*}$ \\
   National Dong Hwa University \\
   Hualien, Taiwan \\
   \texttt{mcshih@gms.ndhu.edu.tw} \\
   \And
   David Sontag \\
   MIT CSAIL \& IMES \\
   Cambridge, MA \\
   \texttt{dsontag@csail.mit.edu} \\
}

\begin{document}

\maketitle

\begin{abstract}
\textit{Randomized Controlled Trials} (RCTs) represent a gold standard when developing policy guidelines. 
However, RCTs are often narrow, and lack data on broader populations of interest.  Causal effects in these populations are often estimated using observational datasets, which may suffer from unobserved confounding and selection bias.  Given a set of observational estimates (e.g. from multiple studies), we propose a meta-algorithm that attempts to reject observational estimates that are biased. We do so using \textit{validation effects}, causal effects that can be inferred from both RCT and observational data. After rejecting estimators that do not pass this test, we generate conservative confidence intervals on the \textit{extrapolated} causal effects for subgroups not observed in the RCT\@. Under the assumption that at least one observational estimator is asymptotically normal and consistent for both the validation and extrapolated effects, we provide guarantees on the coverage probability of the intervals output by our algorithm. 
To facilitate hypothesis testing in settings where causal effect transportation across datasets is necessary, we give conditions under which a doubly-robust estimator of group average treatment effects is asymptotically normal, even when flexible machine learning methods are used for estimation of nuisance parameters.
We illustrate the properties of our approach on semi-synthetic and real world datasets, and show that it compares favorably to standard meta-analysis techniques.
\end{abstract}

\section{Introduction}%
\label{sec:intro}
Policy guidelines often rely on conclusions from Randomized Controlled Trials (RCTs), whether considering treatment decisions in healthcare, classroom interventions in education, or social programs in economics \citep{keum2019vitamin,cloyd2020neoadjuvant, prete2018robotic}. In healthcare, when a target population has reasonable overlap with the inclusion criteria of RCTs, current clinical treatment guidelines rely primarily on RCTs \citep{guyatt2008grade,guyatt2008quality}. For target populations not well-represented in RCTs, observational studies are often used to infer treatment effects. However, different observational estimates can give conflicting conclusions. We give an example of this tension when looking at a new chemotherapy for multiple myeloma.

\begin{example}[\textit{Carfilzomib-based Combination Therapy for Newly Diagnosed Multiple Myeloma (NDMM)}]\label{ex:multiple_myeloma}
Until 2020, the effect of Carfilzomib-based combination therapy in the NDMM subpopulation had not been studied via an RCT. However, a trial (ASPIRE) in 2015 measured the effect of Carfilzomib-based therapy on survival in Relapsed \& Refractory Multiple Myeloma (RRMM) patients \citep{stewart2015carfilzomib}. The CoMMpass trial, an observational dataset, was also available in which the Carfilzomib regimen was given to both NDMM and RRMM patients \citep{usrelating}. Several analyses on the CoMMpass dataset to estimate the effect of Carfilzomib-based therapy on NDMM patients %
led to different, sometimes opposing, conclusions on the benefit of the therapy in this subpopulation \citep{li2018comparing,landgren2018carfilzomib}.
\end{example}

A traditional meta-analysis approach would combine observational estimates under the assumption that differences arise only due to random variation, and not e.g., differences in confounding bias \citep[Section 10.10.4.1]{higgins2019cochrane}. This is unlikely to be true in practice. For instance, in Example~\ref{ex:multiple_myeloma}, the two studies in question made different choices in e.g., how to adjust for confounders.
\textit{In this paper, we relax the assumption that all observational estimates are valid.} Instead, we assume that at least one observational estimate is valid across all subpopulations. In the context of Example~\ref{ex:multiple_myeloma}, we might assume that at least one of the candidate observational studies yields consistent and asymptotically normal estimates of the effects in both the NDMM and RRMM populations. 
While we cannot \textit{verify} that any given estimator is valid for all subpopulations, we can \textit{falsify} this claim of validity if an estimator is inconsistent for the causal effects identified by the RCT (e.g., RRMM).
Hence, we use the term \textit{validation effects} to refer to causal effects in subpopulations that overlap between the observational and randomized datasets (e.g., RRMM), and use the term \textit{extrapolated effects} to refer to those only covered by observational datasets (e.g., NDMM).

We propose a meta-algorithm that combines two key ideas: falsification of estimators, and pessimistic combination of confidence intervals.
We first aim to falsify candidate estimators using hypothesis testing, rejecting those that fail to replicate the RCT estimates of validation effects.
In Section~\ref{sub:motivating_examples}, we motivate this approach with examples of observational estimates based on different causal assumptions, showing that hypothesis tests based on asymptotic normality can be applied even when causal assumptions fail to hold. %
Then, we combine accepted estimators to get confidence intervals on the extrapolated effects.  Since failure to reject does not imply validity,\footnote{For instance, we could fail to reject due to low power, or because falsification is impossible, due to differences in causal structure across subpopulations, as discussed in Appendix~\ref{app:intro-examples}.} we return an interval that contains every confidence interval of the accepted estimators.
We demonstrate theoretically that if \textit{at least one} candidate estimator is consistent for both the validation and extrapolated effects, then the intervals returned by our algorithm provide valid asymptotic coverage of the true effects.

In scenarios where the covariate distribution differs across datasets, estimators that \enquote{transport} the causal effect should be used~\cite{Pearl2014-lm,Dahabreh2019-xg, Dahabreh2020-ua}. Furthermore, in the case of high-dimensional covariates, flexible machine learning methods are required to estimate nuisance functions, which can affect the hypothesis tests due to their slower convergence rates.   In light of this, we adapt estimators of the average treatment effect in this setting to provide estimates of group-wise treatment effects, and show (via the framework of double machine learning \citep{Chernozhukov2018-nk,Semenova2021-zf}) that this estimator enjoys asymptotic normality under mild conditions on convergence rates of the nuisance function estimators.  Our conclusions are supported by semi-synthetic experiments, based on the IHDP dataset, as well as real-world experiments, based on clinical trial and observational data from the Women's Health Initiative (WHI), that demonstrate various characteristics of our meta-algorithm.

\section{Setup and Motivating Examples}%
\label{sec:background}

\subsection{Notation and Assumptions}%
\label{sub:notation_and_setup}

Let $Y \in \cY$ denote an outcome of interest, and $A \in \{0, 1\}$ denote a binary treatment.  We use $Y_a$ to denote the potential outcome of an individual under treatment $A = a$.  We use $X \in \cX$ to denote all other covariates.  To distinguish between different sampling distributions (i.e., datasets), we use the random variable $D \in \{0, \ldots J\}$, where $J \geq 1$ is the number of observational datasets, and $D = 0$ is reserved for the sampling distribution of the randomized trial.  %
We let $\P(Y_1, Y_0, Y, A, X, D)$ denote the joint distribution over all variables, including unobserved potential outcomes.  For instance, $\P(Y_1, Y_0, X \mid D = 0)$ denotes the distribution of potential outcomes and covariates in the RCT\@.  

We seek to estimate conditional average treatment effects for a finite set of $I$ subgroups $\{\cG_i\}_{i=1}^{I}$. We assume subgroups are defined a-priori by a function $G: \cX \mapsto \{1, \ldots, I\}$, such that $G = i$ indicates that $X \in \cG_i$.  We use observational data precisely because not all groups are supported on the RCT dataset. To this end, we use $\cI_{R} = \{i: \P(G = i \mid D = 0) > 0\}$ to denote the set of subgroups supported on the RCT dataset, and we let $\cI_{O}$ denote the complement $\{1, \ldots, I\} \setminus \cI_{R}$.  We use $\abs{\cI_{R}}$ to denote the cardinality of a set, and assume that every observational dataset has support for all groups.

\begin{restatable}[Support]{assumption}{Support}\label{asmp:support}
We assume that $\P(G = i, D = j) > 0$ for all $i \in \{1, \ldots, I\}$ and $j \in \{1, \ldots, J\}$, i.e., all observational datasets ($D \geq 1$) have support for all groups.
\end{restatable}

\begin{definition}[Validation and Extrapolated Effects]\label{def:validation_extrapolation_effects}
We define the group average treatment effect (GATE)\footnote{We use this term in line with the literature \citep{Chernozhukov2017-gc,Jacob2019-bt,Park2019-hr,Semenova2021-zf} and to distinguish it from the CATE function.} as
\begin{equation}\label{eq:definition_gate}
\tau_i \coloneqq \begin{cases}
  \E[Y_1 - Y_0 \mid G = i, D = 0], &\ \text{if } i \in \cI_{R} \\
  \E[Y_1 - Y_0 \mid G = i, D = 1], &\ \text{if } i \in \cI_{O} \\
\end{cases}
\end{equation}
and refer to $\tau_i$ for $i \in \cI_{R}$ as a validation effect, and $\tau_i$ for $i \in \cI_{O}$ as an extrapolated effect.
\end{definition}

Here, we focus on discrete subgroups, in part to reflect the practical reality of comparing RCTs to observational studies, where we may have large observational datasets with rich covariates but only have access to the published results of the RCT, which often provides estimates (with confidence intervals) for subgroup effects but not the raw data itself \citep[Figure 4, for example]{SPRINT_Research_Group2015-gb}. In Def.~\ref{def:validation_extrapolation_effects}, we allow for the fact that different datasets may have different distributions of effect modifiers.  To have a well-defined effect of interest, we have chosen the reference dataset $D = 1$ arbitrarily, but in principle we could choose any of the observational datasets. We discuss further nuances of this definition under Assumption~\ref{asmp:one_good_estimator}.
By Def.~\ref{def:validation_extrapolation_effects}, we often write these effects as a vector $\tau \in \R^{I}$. We use $\tauhk \in \R^{I}$ to denote an estimator, where $k \in \{0, \ldots, K\}$, with $\tauhz$ reserved to denote the estimator derived from the RCT data. The remainder are observational estimators.\footnote{We define $\tauhz$ as a vector in $\R^{I}$ for simplicity of notation, allowing the entries $\tauhzi, i \in \cI_{O}$ to be arbitrary.} In general, we use \enquote{hat} notation to refer to estimators, and refer to their population quantities without a hat. We use $N_k$ to denote the number of samples used by each estimator. %
Throughout, we will assume that the RCT estimator is consistent. %
\begin{assumption}\label{asmp:rct_is_good}
  The RCT estimator $\tauhz$ is a consistent estimator of the (supported) dimensions of $\tau$, such that for each $i \in \cI_{R}$, $\tauhzi$ is consistent for $\taui$.
\end{assumption}
Below, our central assumption states that at least one observational estimator also enjoys consistency.  We discuss examples of specific observational estimators in Section~\ref{sub:motivating_examples}.
\begin{assumption}\label{asmp:one_good_estimator}
  There exists at least one observational estimator $\tauhk \in \R^{I}$, $k \geq 1$ that is a consistent estimator of $\tau \in \R^{I}$, such that for each $i \in \{1, \ldots, I\}$, $\tauhki$ is consistent for $\tau_i$.
\end{assumption}
\begin{remark}
  Assumption~\ref{asmp:one_good_estimator} is our primary non-trivial assumption, and in Appendix~\ref{sec:conditions_for_valid_observational_randomized_comparisons}, we give one example of causal assumptions (for a given observational study) under which the entire GATE vector $\tau$ is \textbf{identifiable} from observational data, and give an \textbf{estimator} of the resulting observational quantity which is asymptotically normal \citep{Pearl2011-rt,Pearl2014-lm,Pearl2015-ta,Dahabreh2020-ua, Degtiar2021-bb}. In order to compare observational estimates with experimental ones, Assumption~\ref{asmp:one_good_estimator} requires not only that the observational data is free of confounding, but also that the causal effect can be transported to the RCT population.  This can be done so long as relevant effect modifiers are observed in both the RCT and observational study, but the latter requirement is satisfied automatically (without requiring RCT data) if e.g., treatment effects are constant within each subgroup $G$, or if the distribution of effect modifiers is the same between the RCT and observational study, in which case $\E[Y_1 - Y_0 \mid D, G] = \E[Y_1 - Y_0 \mid G]$.  This represents one (conservative) failure mode of our approach, in which we may reject an observational estimator due to failures in transportability, even if it yields unbiased estimates of the extrapolated effects. %
\end{remark}
Assumptions~\ref{asmp:rct_is_good} and \ref{asmp:one_good_estimator} imply that there exists an observational estimator $\tauhk$ such that both $\tauhki$ and the RCT estimate $\tauhzi$ are both consistent for the validation effects $\tau_i$, $\forall i \in \cI_R$. To validate this implication in finite samples, we will construct a statistical test to compare $\tauhki$ and $\tauhzi$. %
Our general approach could be modified to use any valid test, %
but to facilitate further analysis, as well as explicit construction of confidence intervals, we additionally assume the following:
\begin{assumption}\label{asmp:all_asymptotic_normal}
  All GATE estimators are pointwise\footnote{Here, \enquote{pointwise} refers to the fact that each subgroup effect estimate is asymptotically normal.} asymptotically normally distributed.
  \begin{equation}
    \sqrt{N_k} (\tauhki - \tauki)/\sigmahki \cid \cN(0, 1)
  \end{equation}
  for all $k \in \{0,,...,K\}$, and for all $i$ in $\cI_R$ if $k = 0$ (the RCT estimator), and otherwise for all $i \in \cI_R \cup \cI_O$.
  Here, $\cid$ denotes convergence in distribution, and $\hat{\sigma}^2_i(k)$ is an estimate of the variance that converges in probability to $\sigma^2_i(k)$, the asymptotic variance of $\sqrt{N_k}(\tauhki - \tauki)$.
\end{assumption}
Assumption~\ref{asmp:all_asymptotic_normal} requires each estimator $\tauhk$ to be consistent and asymptotically normal for some $\tauk$, which may \textbf{not} be equal to $\tau$.  This is not a particularly strong assumption, as we discuss below.

\subsection{Asymptotic Normality of Biased Estimators}%
\label{sub:motivating_examples}

In this section, we give two simple examples to illustrate the principle that multiple estimators $\tauhk$ may be asymptotically normal, even if they are asymptotically biased (i.e., $\tauk \neq \tau$).  In both cases, there is a distinction between the \textit{statistical} assumptions required to obtain asymptotic normality, and the \textit{causal} assumptions required for $\tauk$ to identify the causal effect $\tau$.  For simplicity in both examples, we restrict to the setting of comparing one-dimensional estimates $\tauk \in \R$, which estimate the GATE, $\tau$, in a single group $G = 1$ covered by all datasets.  The statistical claims here also extend to GATE estimation with multiple groups \citep{Semenova2021-zf}.

\begin{example}[Variation in confounding across datasets]\label{ex:double_ml_ate}
  Suppose that there is one estimator of the GATE per observational dataset, and each estimator seeks to estimate the population quantity, $\tauk = \E[g_k(1, X_k) - g_k(0, X_k) \mid G = 1, D = k]$,
  where $X_k$ denotes the controls used in each study, and $g_k(A, X_k) \coloneqq \E[Y \mid A, X_k, D = k]$ and $m_k(X_k) \coloneqq \P(A = 1 \mid X_k, D = k)$. We assume that $\eta < m_k(x) < 1 - \eta$ for some $\eta > 0$ for all $x, k$. Note that $\tauk$ is only a \textit{statistical} quantity: identifying this with the \textit{causal} quantity (the GATE) requires additional assumptions like unconfoundedness, that $Y_a \indep A \mid X_k$ for the given dataset $D$. This assumption may hold for some datasets, but not others, particularly if the set of observed confounders $X_k$ differs across datasets. 
  
  Regardless of the interpretation of $\tauk$, one can construct estimators of it that are consistent and asymptotically normal using flexible machine learning estimators.\footnote{A rich literature focuses on establishing such results, beyond the approach in this example \citep{Athey2016-xa,Farrell2018-ca,Wager2018-ah,Oprescu2019-ui,Athey2019-dl}.} One approach, given in \citet{Chernozhukov2018-nk}, is to use double machine learning (DML), which employs cross-fitting to produce estimates $\tauhk$ based on the doubly-robust score \citep{Robins1995-cn}, while using plug-in estimates $\hat{g}_k, \hat{m}_k$ based on machine learning models.  This approach achieves asymptotic normality, $\sqrt{N_k} (\tauhk - \tauk) / \hat{\sigma}^{2}(k) \cid \cN(0, 1)$,
  under regularity conditions that allow for flexible machine learning estimators that converge at slower than parametric rates, and where $\hat{\sigma}^{2}(k)$ converges in probability to the variance of the doubly robust score \citep[See Theorem 5.1 of ][for additional details]{Chernozhukov2018-nk}. These results hold whether or not $\tauk = \tau$, as discussed in Footnote 9 of \citet{Chernozhukov2018-nk}.  
  For simplicity, we have focused on the case where $\E[Y_1 - Y_0 \mid G = 1]$ is constant across datasets. When this does not hold, certain conditions enable valid transportation of treatment effects across datasets \citep{Degtiar2021-bb} with the use of transported estimators \citep{Dahabreh2020-ua} (see Appendix~\ref{sec:conditions_for_valid_observational_randomized_comparisons} for details).  
\end{example}  

\begin{example}[Selection of Adjustment Strategy]\label{ex:selection_estimators}
  Consider the two causal graphs given in Figure~\ref{fig:two_graphs}, and assume that all variables are binary. Each graph suggests a different identification strategy for the causal effect, $\E[Y \mid do(A = a), G = 1]$.  In Figure~\ref{subfig:my}, this is identified by the (observational) quantity $\E[Y \mid A = a, G = 1]$, and in Figure~\ref{subfig:ay}, by front-door adjustment \citep{Pearl1995-mj} as $\sum_{M} P(M \mid a, G=1) \sum_{A'} \P(Y \mid M, A', G=1) \P(A' \mid G=1)$.
\end{example}

\begin{wrapfigure}{R}{0.55\textwidth}
  \centering
    \begin{subfigure}[t]{0.275\textwidth}
    \centering
    \begin{tikzpicture}[
      obs/.style={circle, draw=gray!90, fill=gray!30, very thick, minimum size=5mm}, 
      uobs/.style={circle, draw=gray!90, fill=gray!10, dotted, minimum size=5mm}, 
      bend angle=30]
      \node[obs] (A) {$A$};
      \node[obs] (M) [right=10pt of A] {$M$};
      \node[obs] (X) [below=8pt of M] {$G$};
      \node[obs] (Y) [right=10pt of M] {$Y$};
      \draw[-latex, thick] (A) -- (M);
      \draw[-latex, thick] (X) -- (A);
      \draw[-latex, thick] (X) -- (M);
      \draw[-latex, thick] (X) -- (Y);
      \draw[-latex, thick] (M) -- (Y);
      \draw[latex-latex, dotted] (M) to[bend left] (Y);
    \end{tikzpicture}
    \caption{}%
    \label{subfig:my}
    \end{subfigure}%
    \begin{subfigure}[t]{0.275\textwidth}
    \centering
      \begin{tikzpicture}[
        obs/.style={circle, draw=gray!90, fill=gray!30, very thick, minimum size=5mm}, 
        uobs/.style={circle, draw=gray!90, fill=gray!10, dotted, minimum size=5mm}, 
        bend angle=30]
      \node[obs] (A) {$A$};
      \node[obs] (M) [right=10pt of A] {$M$};
      \node[obs] (X) [below=8pt of M] {$G$};
      \node[obs] (Y) [right=10pt of M] {$Y$};
      \draw[-latex, thick] (A) -- (M);
      \draw[-latex, thick] (X) -- (A);
      \draw[-latex, thick] (X) -- (M);
      \draw[-latex, thick] (X) -- (Y);
      \draw[-latex, thick] (M) -- (Y);
        \draw[latex-latex, dotted] (A) to[bend left] (Y);
      \end{tikzpicture}
    \caption{}%
    \label{subfig:ay}
    \end{subfigure}
  \caption{(Ex.~\ref{ex:selection_estimators}) In~(\subref{subfig:my}), $M$ and $Y$ are confounded by unobservables (bi-directional dotted arrow). In~(\subref{subfig:ay}), $A$ and $Y$ are confounded, but the causal effect is identified via front-door adjustment.}%
  \label{fig:two_graphs}
\end{wrapfigure}
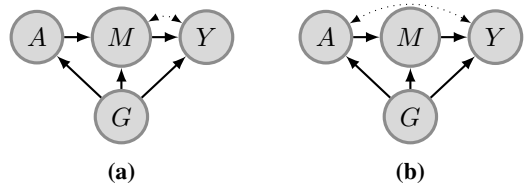 
  These observational quantities will typically differ: the one that represents the true interventional effect depends on which graph reflects the true causal structure. However, in the case where all variables are discrete and low-dimensional, we can still construct asymptotically normal estimators for both observational quantities.\footnote{This follows from the use of maximum likelihood (i.e., empirical counts) for estimating each conditional distribution, and applying the delta method to the front-door estimator.} For more complex settings (e.g., requiring regularized ML models for estimating conditional distributions) asymptotic normality has been established under certain conditions for general graphs \citep{Bhattacharya2020-ig,Jung2021-uu}

\begin{remark}
In each example, there are multiple estimators available, each asymptotically normal under basic statistical assumptions, but potentially biased in the sense that $\tauk \neq \tau$.  In the first example, this bias occurs if $X$ is not sufficient to control for confounding in all observational datasets.  In the second, this bias arises in a given estimator if the causal graph is incorrectly specified. Assumption~\ref{asmp:one_good_estimator} corresponds to assuming that both the statistical assumptions and causal assumptions hold for one of the candidate estimators, e.g., $X$ is sufficient to control for confounding in at least one study (Example~\ref{ex:double_ml_ate}), or that one of the causal graphs is correct (Example~\ref{ex:selection_estimators}).
\end{remark}

\subsection{Asymptotic Normality of GATE Estimators with Transportation}

Example~\ref{ex:double_ml_ate} assumes that $\E[Y_1 - Y_0 \mid G = i]$ is constant across datasets. In practice, it may be necessary to correct for differences (not captured by group indicators) between the observational and RCT populations. There exist estimators for the ATE in this setting under mild additional assumptions \cite{Dahabreh2020-ua, Dahabreh2019-xg}.  These extend in a straightforward way to estimators of the GATE, but proving asymptotic normality is nuanced in high-dimensional settings when using flexible machine learning methods to estimate nuisance functions.  %
For completeness, inspired by \citet{Semenova2021-zf}, we demonstrate that a doubly-robust GATE estimator for this setting is asymptotically normal under reasonable conditions (Assumption~\ref{asmp:bounded_p} to \ref{asmp:nuisance}). Details on the estimator, and the corresponding proof of normality, are given in Appendix~\ref{sec:experimentdetails}, and may be of independent interest.

\subsection{Testing for Bias under Asymptotic Normality}%
\label{sub:testing_for_bias}

Under Assumption~\ref{asmp:all_asymptotic_normal}, each observational estimate $\tauhki$ can be compared to the estimate from the randomized trial $\tauhzi$ for $i \in \cI_{R}$, the groups with common support.  Since the observational and randomized datasets are distinct, we can conclude that each $\tauhki$ is independent of $\tauhzi$, and use this to test for the hypothesis that $\tauki = \tau_i$.
\begin{restatable}{proposition}{ValidityOfNormalTest}\label{prop:validity_of_normal_test}
  For an observational estimator $\tauhk$, assume Assumptions~\ref{asmp:rct_is_good} and \ref{asmp:all_asymptotic_normal} hold.  Furthermore, let $N = N_k + N_0$ with fixed proportions, where $N_k = \rho N, N_0 = (1 - \rho) N$ for $\rho \in (0, 1)$. Define the test statistic 
  \begin{equation}
    \Thnki \coloneqq \frac{\tauhki - \tauhzi - \muki}{\hat{s}} \label{eq:definition_Tn}
  \end{equation}
  where $\hat{s}^2 \coloneqq \frac{\sigmahksi}{N_k} + \frac{\sigmahzsi}{N_0}$ is the estimated variance, and $\muki \coloneqq \tauki - \tau_i$.  This test statistic converges in distribution to a normal distribution as $N \rightarrow \infty$, $\Thnki \cid \cN(0, 1)$.
\end{restatable}
We present the proof for \Cref{prop:validity_of_normal_test} in~\cref{sec:proofs}. This asymptotic normality allows for the construction of simple hypothesis tests.  For instance, one can construct a Wald test for $H_0: \tauki = \tau_i$, with asymptotic level $\alpha$ by setting $\muki = 0$ in Equation~\eqref{eq:definition_Tn} and rejecting $H_0$ whenever, $|\Thnki| > z_{\alpha/2}$,
where $z_{\alpha/2}$ is the $1 - \alpha/2$ quantile of the normal CDF\@. Moreover, the asymptotic power of this test (the probability of correctly rejecting $H_0$) is given by 
\begin{equation}
  1 - \Phi\left(\frac{\muki}{\sigma_{k,0}} + z_{\alpha/2}\right) + \Phi\left(\frac{\muki}{\sigma_{k,0}} - z_{\alpha/2}\right) \label{eq:definition_beta}
\end{equation}
where $\sigma^{2}_{k,0} \coloneqq \frac{\sigmaks_i}{N_k} + \frac{\sigmazsi}{N_0}$ \citep[see Theorems 10.4, 10.6 of ][]{Wasserman2004-qm}.  Likewise, Assumption~\ref{asmp:all_asymptotic_normal} implies an asymptotic $1 - \alpha$ confidence interval for $\tauki$ as
\begin{align}
  [\Lki(\alpha), \Uki(\alpha)] \coloneqq \bigg[&\tauhki - \frac{z_{\alpha/2} \cdot \sigmahki}{\sqrt{N_k}}, \tauhki + \frac{z_{\alpha/2} \cdot \sigmahki}{\sqrt{N_k}}\bigg]\label{eq:definition_lu}
\end{align}

\section{Meta-Algorithm for Conservative Extrapolation}%
\label{sec:algorithm}

\begin{algorithm}[t]
   \caption{Extrapolated Pessimistic Confidence Sets}\label{alg:high_level}
\begin{algorithmic}
   \STATE {\bfseries Input:} Desired coverage $1 - \alpha$. For each $i \in \cI_{R}$, RCT estimate $\tauhzi$ and variance $\sigmahzsi$.  For each $i \in \cI_{R} \cup \cI_{O}$, $K$ candidate estimators $\tauhki$ and variances $\sigmahksi$. Sample sizes $N_0, \ldots, N_K$. 
   \STATE {\bfseries Initialize:} Empty candidate set $\hat{\cC} \gets \varnothing$
   \FOR{$k=1$ {\bfseries to} $K$}
   \STATE{Compute $\Thnki, \forall i \in \cI_{R}$, with $\muki = 0$ (Eq.~\ref{eq:definition_Tn})}
   \STATE \textbf{if} {$\forall i \in \cI_{R}, \abs{\Thnki} \leq z_{\alpha/4\abs{\cI_{R}}}$}, \textbf{then} $\hat{\cC} \gets \hat{\cC} \cup \{k\}$
   \ENDFOR
   \FOR{$i \in \cI_{O}$}
   \STATE $\hat{L}_i \gets \min_{k \in \hat{\cC}} \Lki(\alpha/2)$ and  
   $\hat{U}_i \gets \max_{k \in \hat{\cC}} \Uki(\alpha/2)$ (Eq.~\ref{eq:definition_lu})
   \ENDFOR
   \STATE {\bfseries Return:} $\hat{L}_i, \hat{U}_i$ for each $i \in \cI_{O}$.
\end{algorithmic}
\end{algorithm}

In this section, we more formally introduce our algorithm (\Cref{alg:high_level}).  There are two primary steps: falsification of estimators, and combination of confidence intervals.  First, we attempt to falsify candidate estimators via hypothesis testing, rejecting estimator $k$ whenever we are able to reject the null hypothesis $H_0: \tauki = \tau_i, \forall i \in \cI_{R}$.  We use Bonferroni correction to control the false positive rate of the test.  For the combination of confidence intervals, while we are unlikely to reject the \enquote{correct} estimator if one exists (Assumption~\ref{asmp:one_good_estimator}), we may be unable to reject all \enquote{incorrect} (i.e., biased) estimators. This motivates the combination of confidence intervals (for the extrapolated effects) of the accepted estimators by taking the maximum and minimum bounds over all such intervals.
Our main result characterizes the properties of our procedure, with proof in Appendix~\ref{sec:proofs}.

\begin{restatable}[Properties of Algorithm~\ref{alg:high_level}]{theorem}{Properties}\label{thm:properties_of_procedure}
  Under Assumptions~\ref{asmp:support} and~\ref{asmp:rct_is_good}, the output of Algorithm~\ref{alg:high_level} has the following asymptotic properties as $N \rightarrow \infty$, where $N$ denotes the total sample size, and the samples used for all estimators are of the same order $N_k = \rho_k N_0, \forall k \geq 1$, for some $\rho_k > 0$.
  \begin{enumerate}
    \item Under Assumptions~\ref{asmp:one_good_estimator} and~\ref{asmp:all_asymptotic_normal}, for each $i \in \cI_{O}$,
      \begin{equation}
        \lim_{N \rightarrow \infty} \P(\tau_i \in [\hat{L}_i, \hat{U}_i]) \geq 1 - \alpha
      \end{equation}
    \vspace{-1.5em}
    \item Under Assumption~\ref{asmp:all_asymptotic_normal}, for each estimator where $\tauki \neq \tau_i$ for some $i \in \cI_{R}$, 
      \begin{equation}
        \lim_{N \rightarrow \infty} \P(k \in \hat{\cC}) = 0
      \end{equation}
    \vspace{-2em}
  \end{enumerate}
\end{restatable}

The first point says that for each extrapolated effect $\tau_i$, the coverage of the final confidence interval $[\hat{L}_i, \hat{U}_i]$ is at least $1 - \alpha$ in the limit. It follows from Assumption~\ref{asmp:one_good_estimator} and~\ref{asmp:all_asymptotic_normal} that at least one estimator provides intervals $[\Lki(\alpha/2), \Uki(\alpha/2)]$ that achieve asymptotic coverage of $1 - \alpha/2$.  The result follows from our choice of threshold for the significance test as well as application of union bounds.  
The second point says that we will reject estimators that are not consistent for the validation effects, in the limit. Assumption~\ref{asmp:all_asymptotic_normal} ensures that Proposition~\ref{prop:validity_of_normal_test} holds for all estimators, so that this rejection is a consequence of the asymptotic power in Equation~\eqref{eq:definition_beta}, going to $1$ for a fixed bias as $N \rightarrow \infty$.

\begin{remark}
Equations~\eqref{eq:definition_beta} and~\eqref{eq:definition_lu} are useful for building further intuition.  All of the candidate confidence intervals shrink at a rate of $O(1 / \sqrt{N})$ as the overall sample size increases. For sufficiently large $N$, the width of our generated intervals will depend largely on our power to reject biased estimators, which will be higher for observational estimates with larger biases for validation effects.
\end{remark}

\section{Semi-Synthetic Experiments}%
\label{sec:experiments}

\subsection{Setup of Simulation}\label{sec:ParamsConfig}

We generate semi-synthetic RCTs and observational datasets with covariates from the Infant Health and Development Program (IHDP), a randomized experiment on premature infants assessing the effect of home visits from a trained provider on the future cognitive performance \citep{brooks1992effects}. The outcomes are simulated. 
Our data generation is based on the partial IHDP dataset used in \citep{hill2011bayesian}, which includes $n_0 = 985$ observation, 28 covariates, and a binary treatment variable. 
We construct a scenario where there are four subgroups, defined by the infant's birth weight and maternal marital status: (high [$\ge$ 2000g], married), (low [$<$ 2000g], married), (high, single) and (low, single), which we shorthand as HM, LM, HS and LS. We include all subgroups in the observational studies, but exclude the latter two subgroups for the simulated RCT (i.e. only infants with married mothers are in the RCT).

For each simulated dataset, we generate 1 RCT and $K$ observational studies. For the observational studies, we resample the rows of the IHDP dataset to the desired sample size $n = r\cdot n_0$. We performed weighted sampling to induce a different covariate distribution for observational studies, such that male infants, infants whose mothers smoked, and infants whose mothers worked during pregnancy are less prevalent. Then, we introduce confounding in the observational data, generating $m_c$ continuous confounders and $m_b$ binary confounders. Finally, we simulate outcomes in each dataset, modifying the response surface given in \citet{hill2011bayesian}. In our experiments, we may choose to conceal some confounders in each observational study to mimic unobserved confounding, denoting the number of concealed variables across the $K$ studies as $\mathbf{c_z} = (c_{z1}, c_{z2},..., c_{zK})$. For further details on confounder generation, outcome simulation, and confounder concealment, see \Cref{sec:exp_details}. 
Data generation parameters include $K$, $r$, $m_c$, $m_b$, $\mathbf{c_z}$, and the significance level $\alpha$. By default, we set $K = 5$, $r = 10$, $m_c = 4$, $m_b = 3$, $\mathbf{c_z} = (0, 0, 2, 4, 6)$, and $\alpha$ = 0.05. The full hyperparameter search is provided in Appendix~\ref{sec:exp_details}, and details of hyperparameter tuning can be found in Appendix~\ref{sec:experimentdetails}.

\subsection{Implementation and Evaluation of Meta-Algorithm}
\label{sub:implementation_evaluation}
To implement \Cref{alg:high_level}, we first obtain GATE estimates for the four subgroups and their estimated variances in each observational study, combining techniques from the DML and trasportability literature \citep{Semenova2021-zf, Dahabreh2020-ua}. Estimation details are shown in Appendices~\ref{sec:conditions_for_valid_observational_randomized_comparisons}~and~\ref{sec:experimentdetails}.  For the RCT, we stratify the data into the subgroups HM, LM and estimate the GATEs as the difference of mean outcomes between the treated and untreated. The $z$ tests in \Cref{alg:high_level} are applied to both GATE estimates in the HM and LM subgroups ($|\cI_R|=2$), and the significance level of the tests is set at $\alpha/4$. 

We evaluate performance using two main metrics: (1) the coverage probability of the output confidence intervals (ideally at least $1-\alpha$), and (2) the width of the confidence intervals (narrower is better). In addition to assessing the intervals produced by \cref{alg:high_level}, which we call \textit{Extrapolated Pessimistic Confidence Sets (ExPCS)}, we will evaluate intervals produced by a variant of our algorithm, called \textit{Extrapolated Optimistic Confidence Sets (ExOCS)}. In \textit{ExOCS}, after falsifying estimators, we combine confidence intervals using a random-effects meta-analysis on the non-falsified observational studies. We compare \textit{ExPCS} and \textit{ExOCS} against two baselines. \textit{Meta-Analysis} is a random-effects meta-analysis on all observational studies, as described in \Cref{sec:related_work}, with heterogeneity variance estimated via the DerSimonian-Laird moment method \citep{dersimonian1986meta}. This baseline is the current standard for aggregating observational study results. The second baseline, \textit{Simple Union}, uses the maximum upper bound and minimum lower bound of the $1 - \alpha$ confidence intervals across all observational studies, with no falsification procedure.\footnote{Note that \textit{Simple Union} combines $1 - \alpha$ confidence intervals, while our approach combines $1 - \alpha/2$ confidence intervals to account for the probability of rejecting the \enquote{correct} estimator, if one exists.  As a result, \textit{Simple Union} intervals do not always strictly cover the intervals produced by \textit{ExPCS}.}

\subsection{Results}
\label{sub:semisynthetic}

We perform three semi-synthetic experiments to assess the performance of our proposed meta-algorithm under different scenarios. The first experiment applies our algorithm under the default settings given in \Cref{sec:ParamsConfig}. In the second experiment, we vary the sample size ratio between the observational studies and the original RCT, $r$, from $1$ to $10$. In the third experiment, we vary the proportion of biased observational studies by setting $\mathbf{c_z}$ to be $(0,0,0,0,0)$, $(0,0,0,0,3)$, $(0,0,0,3,3)$ or $(0,3,3,3,3)$, corresponding to $0, 1, 2, 4$ studies being biased out of a total of 5 observational studies. Results for the latter two experiments are shown over 100 simulations of the datasets. Results for all experiments are shown in Figures \ref{fig:Demo}, \ref{fig:Upsample}, and Figure \ref{fig:Biased} in Appendix \ref{sec:add_experiments}, respectively. We observe the following:

\textit{Meta-algorithm produces confidence intervals that cover the true GATE with nominal probability}: 
We demonstrate in \Cref{fig:Demo} the application of our meta-algorithm (\textit{ExPCS}), a variant of it (\textit{ExOCS}), and two other baselines on one dataset. Our goal is to produce narrow confidence intervals that still cover the true GATEs in the extrapolated subgroups.
The confidence intervals of \textit{ExPCS} cover the true GATEs in the extrapolated subgroups with reasonable widths. In contrast, intervals produced by \textit{Meta-Analysis} fail to cover the true GATE in both extrapolated subgroups due to the false assumption of unbiasedness across all studies. The \textit{ExOCS} approach produces narrow intervals for the extrapolated effects, though it barely covers the true effect in the HS subgroup. This hints at the need for a conservative combination of non-falsified studies. However, an overly conservative approach (e.g. \textit{Simple Union}) produces wide intervals that may be of little use for meaningful inference.

Although \textit{Meta-Analysis} produces confidence intervals with inadequate coverage, its intervals for the married subgroups still have considerable overlap with the intervals produced by the RCT.  This suggests that testing the meta-analyzed GATE estimates against the RCT GATE estimates may not be enough to demonstrate their validity. Compared to our \textit{ExPCS} intervals, the lower bounds of the \textit{Simple Union} intervals are higher in several subgroups, since we use a higher confidence level for the candidate intervals corresponding to each study to account for probable error in study falsification.

\begin{figure}[t]
    \centering
    \includegraphics[width = 0.5\linewidth]{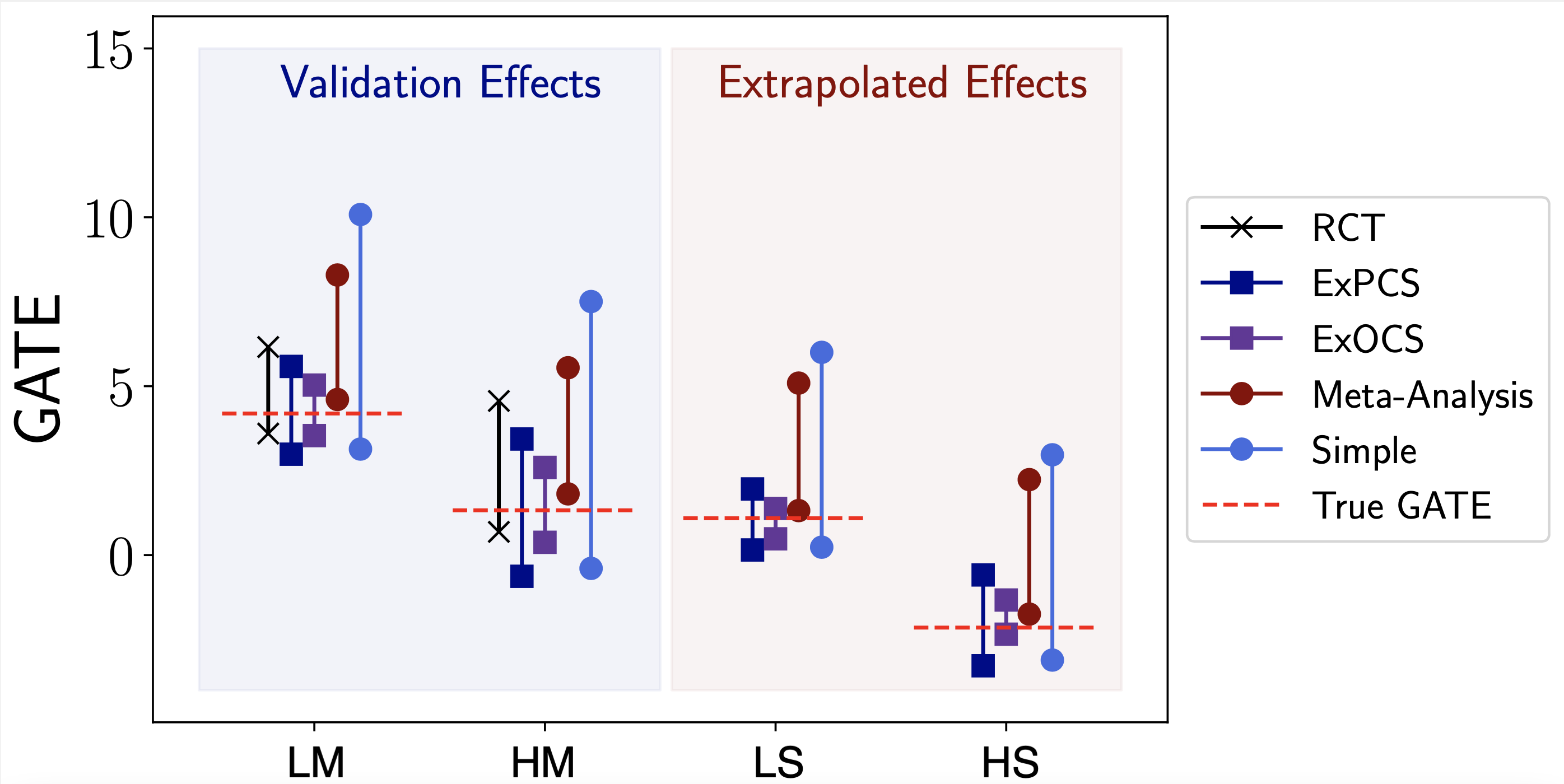}
    \caption{The confidence intervals for group average treatment effects (GATE) within the four subgroups output by our algorithm (\textit{ExPCS}), our algorithm variant (\textit{ExOCS}), random-effects meta-analysis on all observational studies (\textit{Meta-Analysis}), simple union bound on all observational studies (\textit{Simple}), and RCT, for one dataset generated using the default parameter settings laid out in Section~\ref{sec:ParamsConfig}. \textit{LM, HM, LS, HS} represent four subgroups %
    defined in Section~\ref{sec:ParamsConfig}}
    \label{fig:Demo}
\end{figure}

\begin{figure}[t]
    \begin{subfigure}[t]{0.55\textwidth}
        \adjustbox{valign = c}{
            \includegraphics[width = \linewidth]{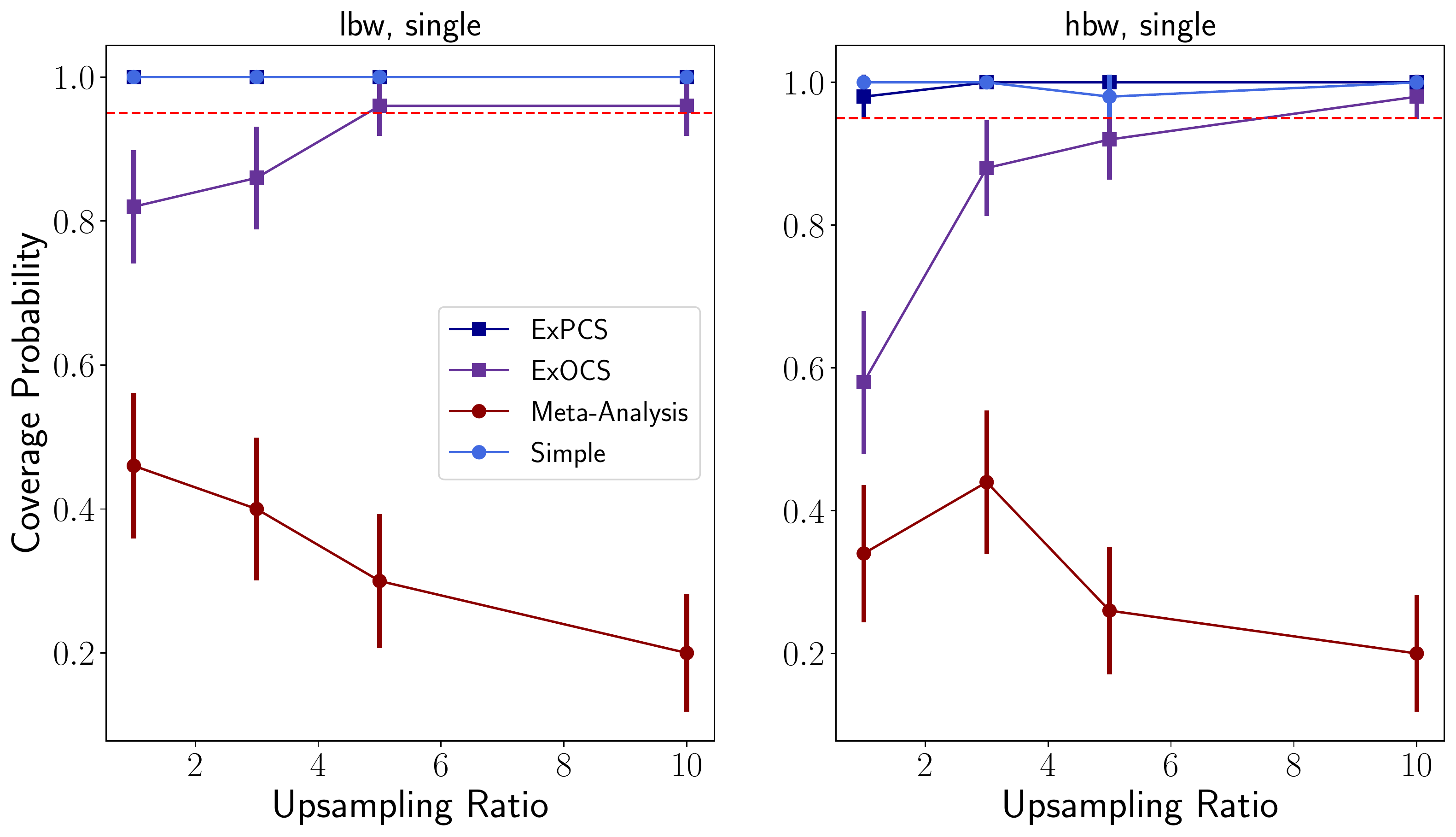}
        }
    \end{subfigure}
    \begin{subfigure}[t]{0.4\textwidth}
        \adjustbox{valign = c}{
            \centering
            \renewcommand\tabularxcolumn[1]{m{#1}}%
            \renewcommand\arraystretch{1.3}
            \setlength\tabcolsep{2pt}%
            \small 
            \begin{tabular}{ccccccc}
                \toprule
                \multicolumn{7}{c}{Mean width of 95\% confidence intervals}\\
                \midrule
                &&& \multicolumn{4}{c}{Upsampling ratio}\\
                \cmidrule{4-7}
                Group&Estimator&& 1 & 3 & 5 & 10\\
                \cmidrule{1-2} \cmidrule{4-7}
                \multirow{2}{*}{\textbf{LS}}&\textit{ExPCS} && 21.00 & 9.51 & 5.53 & 3.58 \\
                &\textit{Simple} && 21.00 & 10.20 & 6.95 & 5.69 \\
                \midrule
                \cmidrule{1-2} \cmidrule{4-7}
                \multirow{2}{*}{\textbf{HS}}&\textit{ExPCS} && 25.00 & 8.29 & 5.36 & 4.34 \\
                &\textit{Simple} && 26.00 & 9.15 & 6.90 & 6.33 \\
                \bottomrule
            \end{tabular}
            
        }
    \end{subfigure}
    \caption{Coverage probabilities of confidence intervals shown as a function of the size of the observational studies relative to the RCT. Dotted red lines stand for 95\% coverage. Vertical bars are the 95\% confidence intervals of the coverage probabilities.  LS / HS stand for groups with low / high birth weight and single mother.  Between \textit{ExPCS} and \textit{Simple} which have adequate coverage, \textit{ExPCS} generally has narrower intervals.}
    \label{fig:Upsample}
\end{figure}

\textit{An analysis of increasing observational study size}: 
In \cref{fig:Upsample}, we find that the coverage of the \textit{Meta-Analysis} intervals is quite low across all sample sizes and particularly decreases at higher sample sizes. This result is intuitive, as three out of five studies are biased, meaning that meta-analysis will converge to a biased estimate as the amount of data increases. One could attempt to fix this issue through \textit{ExOCS}, which does meta-analysis after falsification. However, \textit{ExOCS} has poor coverage when the sample size of the observational studies is small, since the falsification tests are underpowered (evidenced by the high probability of selecting biased studies in \cref{sec:add_experiments}, \Cref{tab:select_bias}). 
Both \textit{ExPCS} and the \textit{Simple Union} intervals have adequate coverage across all sample sizes. %
However, the widths of the intervals reported at the bottom of \Cref{fig:Upsample} show that \textit{ExPCS} intervals are narrower when there is adequate power, i.e. at higher sample sizes. Ultimately, \textit{ExPCS} will tend to provide intervals that cover the true effect regardless of sample size, and in the case we have sufficient power, these intervals will both have good coverage and narrower width, allowing for more meaningful inference. 

\section{Women's Health Initiative (WHI) Experiments}%
\label{sec:whi-experiments}

In order to assess our approach in a real-world setting, we use clinical trial and observational data available from the WHI. Each subgroup is supported in both RCT and observational data, which proves useful for evaluation. At a high level, we ``hide'' some number of subgroups from the RCT, estimate a confidence interval of the effect estimate using our algorithm on the remaining data, and compare the result to the hidden RCT estimate. We do this over a large set of possible ``held-out'' subgroups, yielding >2000 different scenarios on which to test our approach. Because the original observational datasets replicate the RCT results fairly well using standard methods, we create additional ``biased'' datasets by sub-selecting the original observational dataset in a way that induces selection bias. We evaluate each method, for each held-out subgroup, according to the length of the intervals as well as coverage of the RCT point estimates. Below, we describe the specifics of the data, the experimental setup, and the main results of the analysis. For additional details on data preprocessing, setup, and evaluation, see Appendix \ref{sec:whi_exp_details}. 

\subsection{Setup} 
The Postmenopausal Hormone Therapy (PHT) trial, i.e. the RCT used in this analysis, was run on postmenopausal women aged 50-79 years who had an intact uterus. It studied the effect of hormone combination therapy on several types of cancers, cardiovascular events, and fractures. The observational study (OS) was run in parallel, had a similar follow-up time to the RCT, and tracked similar outcomes.
In our analysis, we use a composite outcome, where $Y = 1$ if any of the following events are \textbf{observed} to occur in the first 7 years of follow-up, and $Y = 0$ otherwise: coronary heart disease, stroke, pulmonary embolism, endometrial cancer, colorectal cancer, hip fracture, and death due to other causes.  This represents a binarization of the ``global index'' time-to-event outcome from the original study, where $Y = 0$ could also occur due to censoring.
We establish treatment and control groups in the OS based on explicit confirmation or denial of usage of both estrogen and progesterone in the first three years. We use only covariates measured in both the RCT and OS to simplify analysis. 

\subsection{Evaluation}

Our empirical evaluation consists of several steps. In the first step, we replicate the principal results from the PHT trial, given in Table 2 of \citep{rossouw2002risks}, by fitting a doubly robust estimator (of the style given in Appendix \ref{sec:experimentdetails}) on the WHI OS data. Then, while treating the WHI OS dataset as the ``unbiased'' observational dataset, we simulate additional ``biased'' observational datasets by inducing selection bias into the WHI OS. The exact mechanism of selection bias and its clinical intuition is given in Appendix \ref{sec:whi_exp_details}. Importantly, this is the only part of the evaluation that involves any simulation. %

The second step is to construct a large suite of tasks on which to evaluate our method, by considering different sets of validation-extrapolation subgroups. To construct the subgroups, we consider all pairs of a selected set of binary covariates (see Appendix~\ref{sec:cov-task}), where each pair defines four subgroups. For example, one covariate pair is (\enquote{current smoker}, \enquote{currently drinks alcohol}). We treat two of the subgroups as validation subgroups and two as extrapolated subgroups. For the latter groups, we apply our algorithm without access to the RCT data, and only use the RCT data for final evaluation. The total number of covariate pairs is 592, leading to 1184 distinct ``tasks'' (i.e., extrapolated groups). For each task, we evaluate ExPCS (our method), ExOCS, Simple, and Meta-Analysis (described in~\Cref{sub:implementation_evaluation}). Additionally, we evaluate an ``oracle'' method, which is identical to ExPCS, except that it always selects only the original observational study (i.e. the base WHI OS to which we have not added any selection bias).  For each method, we compute the following metrics, averaged across all tasks – \textbf{Length}: length of the confidence interval, \textbf{Coverage}: percentage of tasks where the interval covers the RCT point estimate. In addition, we report the \textbf{Unbiased OS Percentage}: the percentage of tasks where the ExPCS approach retains the unbiased study after the falsification step.

\begin{wraptable}{r}{7.5 cm}
\centering
\vspace{-6mm}
{\footnotesize
\begin{tabular}{lccc}
\toprule
\toprule
               & \textbf{Coverage}     &\textbf{ Length}     & \textbf{OS $\%$}  \\ \midrule
\quad  Simple  & 0.39 & 0.416 & –  \\
\quad  Meta-Analysis  & 0.03 & 0.260 & – \\
\quad  ExOCS  & 0.28 & 0.058 & – \\
\quad  \textbf{ExPCS (ours)}  & 0.45 & 0.081 & 0.99 \\
\midrule
\quad Oracle  & 0.44 & 0.068 & – \\

\bottomrule
\bottomrule
\end{tabular}}
\caption{\footnotesize Coverage, length, and unbiased OS $\%$ of ExPCS and baselines. ExPCS achieves comparable coverage to the oracle method with highly efficient intervals. Additionally, we do not reject the unbiased OS in $99\%$ of the tasks.}
  \vspace{-5mm}
 \label{tab:whi_table}
\end{wraptable}

\subsection{Results}
Table~\ref{tab:whi_table} reports the metrics above, averaged across all extrapolated subgroups. 

\textit{Compared to the “simple” baseline, our approach has better coverage with much shorter confidence intervals.} Our falsification procedure retains the unbiased observational study 99\% of the time, yielding near-oracle coverage rates, but produces substantially shorter intervals than the \enquote{simple} baseline. Recall that the simple baseline takes a union over all $1 - \alpha$ intervals estimated from each observational dataset, while ExPCS takes a union of a smaller number of slightly wider ($1 - \alpha/2$) confidence intervals.  

\textit{Compared to the Meta-Analysis and ExOCS baselines, we achieve comparable (or much better) length with substantially better coverage.} In particular, compared to meta-analysis, we achieve tighter intervals and also cover the RCT estimate with higher frequency. This result is intuitive, since one will get a biased estimate if biased observational studies are included in the meta-analysis. Additionally, conservatively combining the non-falsified estimates (as opposed to \textit{ExOCS}, which does a meta-analysis on the non-falsified estimates) is important to achieve good coverage (0.45 vs 0.28).

\textit{We get comparable coverage and interval lengths to the oracle method.} Our coverage rate is nearly identical (0.45) to that of the oracle method (0.44), with intervals that are marginally wider (0.081 vs. 0.068). Our slightly improved coverage is possible due to the wider intervals. Note that our measure of “coverage” may be pessimistic, because we track coverage of the RCT point estimate, as opposed to the true causal effect (which is unknown), and the confidence intervals are designed to cover the latter. Indeed, we report the oracle method precisely as a means of providing a more suitable comparison.
Overall, our real-world results suggest that our method of falsification followed by a conservative combination of intervals may be useful for biostatisticians and clinicians when doing meta-analyses. 

\section{Related Work}
\label{sec:related_work}
\textbf{Meta-analysis for combining observational estimates} %
Among the quantitative approaches for meta-analysis to account for potential bias, our \textit{Meta-Analysis} baseline is standard for meta-analysis of observational data \citep{higgins2019cochrane} to account for heterogeneity. Allowing for heterogeneity of treatment effects among studies produces wider confidence intervals and thus more conservative inference. If additional study-level covariates are available (e.g. study designs, drop-out rate), several approaches aim to adjust for potential bias, either by modeling the bias magnitude
\citep{eddy1990introduction, wolpert2004adjusted, Anglemyer2014-et, greenland2005multiple}, down-weighting studies with higher risk of bias  \citep{ibrahim2000power, neuenschwander2009note}, or using Bayesian hierarchical regression to account for difference between subgroups of studies \citep{prevost2000hierarchical, welton2009models}. Our work differs from these approaches, in that (1) we use information from outside the population of interest to assess bias, and (2) we do not place any assumptions on the patterns of bias across studies.

\textbf{Partial identification and sensitivity analysis} These methods seek to place bounds on causal effects when they cannot be point-identified.  Our method can be seen as an alternative way of doing so, with a fundamentally different type of assumption.  Methods for partial identification rely on having discrete variables and a known causal graph (typically including unobserved confounders) \citep[Section 9]{Duarte2021-ld}. %
Methods for sensitivity analysis, on the other hand, translate assumptions about the strength and nature of unobserved confounding into bounds on causal effects \citep{Rosenbaum1983-sy,rosenbaum2010design,Yadlowsky2018-zy}.  %
In contrast, we do not make any such assumptions, e.g., we allow for continuous variables, and when some candidate estimators are biased due to unmeasured confounding, we do not place any limit a-priori on the bias. %
An extended related work is given in Appendix~\ref{sec:app_related_work}.

\section{Discussion and Limitations}%
\label{sec:conclusion}
We have presented a meta-algorithm that constructs conservative confidence intervals for group average treatment effects of subgroups that are not represented in RCTs, but are represented in observational studies.  Under the assumption that there exists at least one candidate estimator that is asymptotically normal and consistent for both the validation and extrapolated effects, these intervals will achieve the correct asymptotic coverage of the true effect.
However, our method is not without limitations. Most notably, we may fail to reject the null hypothesis due to low power, e.g., when an observational estimate $\tauhk$ has high variance. In practice, we expect that our approach will be most useful when the observational studies in question have large sample sizes, leading to higher-precision estimates of potential bias, and smaller confidence intervals on the extrapolated effects. %
Our hope is that methods such as ours will lead to higher confidence in observational estimates when RCT data is available to falsify observational studies that do not replicate known causal relationships.
Finally, great care should be taken to appropriately validate and soundly interpret the results of our method in practice, especially with more sensitive subgroups (e.g. with respect to race or gender).

\newpage 
\textbf{Acknowledgments}: We would like to thank Ahmed Alaa, Hunter Lang, Christina Ji, Hussein Mozannar, and other members of the Clincal Machine Learning group for helpful discussions and valuable feedback on the manuscript. ZH was supported by an ASPIRE award from The Mark Foundation for Cancer Research and by the National Cancer Institute of the National Institutes of Health under Award Number F30CA268631. The content is solely the responsibility of the authors and does not necessarily represent the official views of the National Institutes of Health. MO and DS were supported in part by Office of Naval Research Award No. N00014-21- 1-2807. MCS was supported by the LEAP program from the Ministry of Science and Technology in Taiwan. This manuscript was prepared using WHI-CTOS Research Materials obtained from the National Heart, Lung, and Blood Institute (NHLBI) Biologic Specimen and Data Repository Information Coordinating Center and does not necessarily reflect the opinions or views of the WHI-CTOS or the NHLBI.

\bibliography{ref}
\bibliographystyle{plainnat}

\newpage
\appendix
\onecolumn

\section{When can biased estimators be falsified?}%
\label{app:intro-examples}

As discussed in Examples~\ref{ex:double_ml_ate} and~\ref{ex:selection_estimators}, we imagine that observational estimators differ in a few possible ways. They may represent the same identification strategy applied to different datasets, different identification strategies applied to the same dataset (e.g., different choices of confounders), or some combination of the two.

Assumption~\ref{asmp:one_good_estimator} states that there exists a consistent and asymptotically normal observational estimator for $\tau$, as defined in Def.~\ref{eq:definition_gate}. This is a fundamental assumption in our work, and so we build additional intuition for when we might expect this condition to hold, and when we might be able to falsify this assumption. In this section, we give basic intuition regarding patterns of confounding, and in Section~\ref{sec:conditions_for_valid_observational_randomized_comparisons}, we discuss issues of transportability.

In Example~\ref{ex:intro_example1}, we give a simple example where the causal graph is consistent across two subgroups, and where an estimator must control for all confounders to get consistent estimates of the GATE in either subgroup.  In this setting, falsification is possible. On the other hand, in Example~\ref{ex:intro_example2}, we give a counterexample, where there are multiple estimators that can deliver consistent estimates of the GATE on the RCT subpopulation, but only one provides consistent estimates across all subpopulations.

 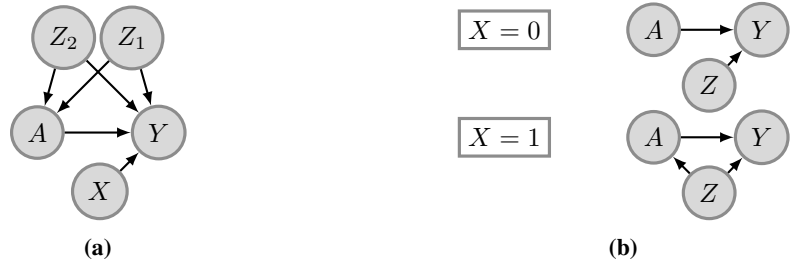
\begin{figure}[t]
  \centering
    \begin{subfigure}[t]{0.5\textwidth}
    \centering
    \begin{tikzpicture}[
      obs/.style={circle, draw=gray!90, fill=gray!30, very thick, minimum size=5mm}, 
      uobs/.style={circle, draw=gray!90, fill=gray!10, dotted, minimum size=5mm}, 
      bend angle=30]
      \node[obs] (A) {$A$};
      \node[obs] (Y) [right=25pt of A] {$Y$};
      \node[obs] (Z1) [above right=28pt of A] {$Z_1$};
      \node[obs] (Z2) [above left=28    pt of Y] {$Z_2$};
      \node[obs] (X) [below left=10pt of Y] {$X$};
      \draw[-latex, thick] (A) -- (Y);
      \draw[-latex, thick] (Z1) -- (A);
      \draw[-latex, thick] (Z1) -- (Y);
      \draw[-latex, thick] (Z2) -- (A);
      \draw[-latex, thick] (Z2) -- (Y);
      \draw[-latex, thick] (X) -- (Y);
    \end{tikzpicture}
    \caption{}%
    \label{subfig:ex1}
    \end{subfigure}%
    \begin{subfigure}[t]{0.5\textwidth}
    \centering
      \begin{tikzpicture}[
        obs/.style={circle, draw=gray!90, fill=gray!30, very thick, minimum size=5mm}, 
        note/.style={rectangle, draw=gray!90, very thick, minimum size=5mm}, 
        uobs/.style={circle, draw=gray!90, fill=gray!10, dotted, minimum size=5mm}, 
        bend angle=30]
      \node[obs] (A) {$A$};
      \node[note] (An) [left=of A] {$X = 0$};
      \node[obs] (Y) [right=20pt of A] {$Y$};
      \node[obs] (Z) [below right=9pt of A] {$Z$};
      \node[obs] (A2) [below=20pt of A] {$A$};
      \node[note] (An2) [left=of A2] {$X = 1$};
      \node[obs] (Y2) [right=20pt of A2] {$Y$};
      \node[obs] (Z2) [below right=9pt of A2] {$Z$};
      \draw[-latex, thick] (A) -- (Y);
      \draw[-latex, thick] (Z) -- (Y);
      \draw[-latex, thick] (A2) -- (Y2);
      \draw[-latex, thick] (Z2) -- (Y2);
      \draw[-latex, thick] (Z2) -- (A2);
      \end{tikzpicture}
    \caption{}%
    \label{subfig:ex2}
    \end{subfigure}
  \caption{Example~\ref{ex:intro_example1} is depicted in~(\subref{subfig:ex1}), and Example~\ref{ex:intro_example2} in~(\subref{subfig:ex2})}%
  \label{fig:intro_example_graphs}
\end{figure}

\begin{example}[Consistent confounding across subgroups]\label{ex:intro_example1}
In the causal graph shown in \cref{subfig:ex1}, there are two sets of confounders, $Z_1,Z_2$, a binary treatment variable $A$, a binary subgroup variable $X$, and the outcome $Y$. We assume a linear outcome model, whereby $E[Y|X,Z_1,Z_2,A] = \alpha+\beta X + \gamma_1 A X + \gamma_2 A (1-X) + \delta_1 Z_1 + \delta_2 Z_2$. Note that the true group average treatment effect (GATE) for the two subgroups are, $\text{GATE}(X=0) = \gamma_2; \text{GATE}(X=1) = \gamma_1$. It is straightforward to show that not conditioning on the full set of confounders will lead to biased GATE estimates for both subgroups, whereas conditioning on both $Z_1$ and $Z_2$ will lead to consistent estimates for both subgroups. %
\end{example}

\begin{example}[Selective confounding by subgroup]\label{ex:intro_example2}
Let there be two subgroups, $X=0$ and $X=1$, with the former having support in both RCT and observational studies and the latter having support in only observational data. Now, suppose we had the following treatment assignment mechanism, $p(A=1|X,Z) = f(Z)\cdot \mathbf{1}(X = 1) + c\cdot \mathbf{1}(X = 0)$, where $Z$ is a set of confounders, $f$ is a nonlinear function of $Z$, and $c$ is a constant. A candidate estimator that does not condition on $Z$ would be able to get consistent estimates for the validation effect but not the extrapolated effect. On the other hand, conditioning on $Z$ would allow for consistent estimates on both validation and extrapolated effects.
\end{example}

\section{Conditions for valid observational / randomized comparisons}%
\label{sec:conditions_for_valid_observational_randomized_comparisons}

Recall that we had defined the group average treatment effect (GATE) as follows in Equation~\eqref{eq:definition_gate}
\begin{equation}
\tau_i \coloneqq \begin{cases}
  \E[Y_1 - Y_0 \mid G = i, D = 0], &\ \text{if } i \in \cI_{R} \\
  \E[Y_1 - Y_0 \mid G = i, D = 1], &\ \text{if } i \in \cI_{O} \\
\end{cases},
\end{equation}
and refer to $\tau_i$ for $i \in \cI_{R}$ as a validation effect, and $\tau_i$ for $i \in \cI_{O}$ as an extrapolated effect.  In this section, we discuss sufficient conditions under which these causal effects are identifiable from observational data drawn from a distribution $D = k$, and give examples of doubly-robust estimators of these quantities. These assumptions cover both comparisons of the observational studies to the randomized trial (used for validation), as well as the normalization of observational estimates (used for confidence intervals on the extrapolated effects).

Our goal in presenting these results is to build intuition in this setting for when we might expect a consistent observational estimator to exist across all groups.  This is a well-studied topic, often in the context of generalizing effect estimates from randomized trials to other supported populations (e.g., all trial-eligible individuals). We primarily make use of results in that literature to build intuition here, pointing the reader to \citet{Degtiar2021-bb} for a recent review whose presentation we largely mirror, with modifications to account for our notation.

\subsection{Identification} 
First, we state standard assumptions under which the GATE in the observational population for $D = k$, 
\begin{equation}
    \E[Y_1 - Y_0 \mid G = i, D = k],
\end{equation}
is identifiable from data in the dataset $D = k$, with notation adapted to our setting.
\begin{assumption}\label{asmp:internal_validity} The following conditions hold for the distribution $\P(\cdot \mid D = k)$:
  \begin{enumerate}
      \item \textit{Conditional Exchangeability over $A$}: $Y_a \indep A \mid X, D = k$ for all treatments $a$.
      \item \textit{Positivity of Treatment Assignment}: $\P(X = x \mid D = k) > 0 \implies \P(A = a \mid X = x, D = k) > 0$ for all $a$.
      \item \textit{Consistency}: $A = a \implies Y_a = Y$
  \end{enumerate}
\end{assumption}
These causal assumptions ensure that the ATE and CATE can be identified from observational data for the observational population and are standard in the causal inference literature \citep{Imbens2015-il}.  In order to transport these estimates to the RCT population (or from one observational dataset to another), we require additional assumptions.  Next, we give assumptions under which these estimates can be transported to another population $D = k'$, where in our case $k' \in \{0, 1\}$.
\begin{assumption}\label{asmp:external_validity} Let $k$ correspond to a source population, and $k'$ correspond to the target population. Conditioned on the event $D \in \{k, k'\}$, define the random variable $S = 1$ if $D = k$ and $S = 0$ otherwise.  Then let the following hold, on the distribution $\P(\cdot \mid D \in \{k, k'\})$.
  \begin{enumerate}
      \item \textit{Conditional Exchangeability over $S$}: $Y_a \indep S \mid X$ for all treatments $a$.
      \item \textit{Positivity of Selection}: $\P(X = x) > 0 \implies \P(S = 1 \mid X = x) > 0$ almost surely over $X$ for all $a$.
      \item \textit{Consistency}: $S = s$ and $A = a \implies Y_a = Y$
  \end{enumerate}
\end{assumption}
Here, we note that this introduces non-trivial additional assumptions. Most notably, we require that the potential outcomes are independent of the dataset, given $X$.  This would be violated, for instance, if the distribution of unobservable effect modifiers differs between different observational studies.  As a result, we note that it is possible for an observational study to fail to replicate the RCT results due to failures of transportability (failure of Assumption~\ref{asmp:external_validity}) even if it has \enquote{internal validity}, allowing for identification of the causal effect in the population $D = k$.  There also exists a large body of work on identifying transportable causal effects via causal graphs \citep{Pearl2011-rt,Pearl2014-lm,Pearl2015-ta}.

\subsection{Estimation of the ATE in the target population}

Regarding estimation, \citet{Dahabreh2020-ua} consider the problem of transporting average treatment effects from randomized trials to observational studies, under Assumption~\ref{asmp:internal_validity} with $k = 0$ and Assumption~\ref{asmp:external_validity} with $k = 0, k' = 1$.  These assumptions admit identification of the potential outcomes means as follows (see Section 4.2 of \citet{Dahabreh2020-ua})
\begin{equation}
   \E[Y_a \mid S = 0] = \E[\E[Y \mid X, S = 1, A = a] \mid S = 0]\label{eq:reweight_id}
\end{equation}
where the outer expectations are over $\P(X \mid S = 0)$, i.e., the covariate distribution of the target population.  \citet{Dahabreh2019-xg} give a doubly robust estimator for the statistical quantity on the right-hand side as the empirical expectation of the following pseudo-outcome (see Equation A.13 of \citet{Dahabreh2020-ua})
\begin{align}
\hat{\mu}(a) &= \frac{1}{n} \sum_{i=1}^{n} Y^{a}_i(\hat{\eta}, \hat{\pi})\label{eq:reweighting_empirical_average}
\end{align}
where $n$ is the total samples in both the source $S = 1$ and target $S = 0$ samples, and where
\begin{align}
Y^{a}_i(\hat{\eta}, \hat{\pi}) &\coloneqq \frac{1}{\hat{\pi}} \left(\1{S_i = 1, A_i = a} \cdot \frac{1 - \hat{p}(X_i)}{\hat{p}(X_i) \hat{e}_a(X_i)} \cdot  \{ Y_i - \hat{g}_a(X_i)\} + (1 - S_i) \hat{g}_a(X_i)\right)\label{eq:definition_transport_score}.
\end{align}
In \Cref{eq:definition_transport_score}, $\hat{\eta} \coloneqq (\hat{g}_a, \hat{e}_a, \hat{p})$, and $\hat{\pi} \coloneqq n^{-1} \sum_{i = 1}^{n} \1{S_i = 0}$ is an estimate of $\P(S = 0)$, $\hat{g}_a(X)$ is an estimate of the mean conditional outcome $\E[Y \mid A = a, S = 1, X]$, $\hat{p}(X)$ is an estimate of the selection probability $\P(S = 1 \mid X)$, and $\hat{e}_a(X)$ is an estimate of the propensity score $\P(A = a \mid S = 1, X)$.  \citet{Dahabreh2019-xg} derives precise asymptotic properties of this estimator, which is asymptotically normal and consistent for the observational quantity on the right-hand side of Equation~\eqref{eq:reweight_id}.  In particular, this estimator is doubly-robust in the sense that it is consistent if either $\hat{p}(X)$ or $\hat{g}_a(X)$ is consistent, but requires consistency of $\hat{e}_a(X)$.  It also enjoys the rate double-robustness property, retaining consistency and asymptotic normality even if the estimators for $\hat{p}, \hat{g}$ converge at slower than parametric rates, and allows for the same cross-fitting schemes used in the Double ML \citep{Chernozhukov2018-nk} literature for relaxing Donsker conditions.

Note that the average treatment effect in this setting can be estimated by the following contrast, which is similarly an empirical expectation of a pseudo-outcome
\begin{equation}\label{eq:reweighting_empirical_contrast}
\hat{\mu}(1) - \hat{\mu}(0) = \frac{1}{n} \sum_{i=1}^{n} Y^{1}_i(\hat{\eta}, \hat{\pi}) - Y^{0}_i(\hat{\eta}, \hat{\pi}) = \frac{1}{n} \sum_{i=1}^{n} Y_i(\hat{\eta}, \hat{\pi}),
\end{equation}
where $Y_i(\hat{\eta}, \hat{\pi}) \coloneqq Y^{1}_i(\hat{\eta}, \hat{\pi}) - Y^{0}_i(\hat{\eta}, \hat{\pi})$. Furthermore, the variance of these estimates can be estimated using either sandwich estimators from M-estimation theory \citep{Stefanski2002-ci}, or via bootstrap methods. We refer the reader to Sections 5.3, 5.4 and Appendix A.4 of \citep{Dahabreh2020-ua} for more details.  

\section{Estimation and comparison of GATE in semi-synthetic experiments}%
\label{sec:experimentdetails}

In Sections~\ref{sub:motivating_examples} and~\ref{sec:conditions_for_valid_observational_randomized_comparisons}, we discuss several estimators for average treatment effects (ATEs) that are known to be asymptotically normal, such as the double ML estimator discussed in Example~\ref{ex:double_ml_ate} or the doubly-robust estimator in Section~\ref{sec:conditions_for_valid_observational_randomized_comparisons}. 

Given a fixed set of discrete subgroups, one could analyze each subgroup independently and apply such estimators directly, since the ATE in each subgroup is precisely the GATE. This would be a straightforward way to ensure that the same formal guarantees hold regarding asymptotic normality. While this approach would be feasible in our experimental setting, due to the small number of groups, it is less practical in general, especially with a larger number of groups, since information cannot be shared across nuisance models such as $\hat{g}_a, \hat{e}_a, \hat{p}$ discussed in Section~\ref{sec:conditions_for_valid_observational_randomized_comparisons}.

In an effort to emulate a more realistic setting, we take a slightly different approach in the semi-synthetic experiments.  We draw inspiration from the double ML approach given in \citet{Semenova2021-zf} for GATE estimation, while taking into consideration the transportation of causal effects in the sense of Section~\ref{sec:conditions_for_valid_observational_randomized_comparisons}. Note that in \citet{Semenova2021-zf}, the required assumptions and proofs for asymptotic normality of estimators are provided on a case-by-case basis, which does not include our case with transportation.  Therefore, in the following we will briefly describe their approach, then show how we construct our GATE estimators and provide the required assumptions for their asymptotic normality.

\citet{Semenova2021-zf} focuses on the setting where there exists some pseudo-outcome / signal, $Y(\eta)$, and where one is interested in summarizing the function, $\tau(x) = \E[Y(\eta) \mid X = x]$, with a linear regression function (in the simplest case, a set of group indicators).  When $Y(\eta)$ is the doubly-robust score \citep{robins1994estimation, Robins1995-cn} (see Equation~\eqref{eq:doubly_robust_score}), $\tau(x)$ is equal to the CATE function, and the best approximation by group indicators gives the GATE. 

Our general procedure is as follows: for estimation of $\tauhk$ and the respective variances, we construct a score function / pseudo-outcome, $\tilde{Y}$, whose empirical conditional expectation (in each group) provides an estimate of the GATE, and whose empirical variance we use as an estimate of the variance.  We describe this procedure in more detail below.  Throughout, $X$ should be taken to refer to the covariates that are observed in a given observational study.

\paragraph{Comparing Validation Effect Estimates} In our simulation setup, all of the observational datasets are drawn from a common distribution, which differs from the RCT distribution, requiring the use of the techniques and assumptions discussed in Section~\ref{sec:conditions_for_valid_observational_randomized_comparisons} to estimate the GATE, $\tau_i = \E[Y_1 - Y_0 \mid G = i, D = 0]$, using data from the observational distributions.

To generate the observational estimates $\tauhki, \sigmahksi$ in this setting, we cannot simply take empirical conditional expectation / variance of the score function given in Equation~\ref{eq:reweighting_empirical_contrast}.  Rather, the GATE is identified under Assumptions~\ref{asmp:internal_validity} and~\ref{asmp:external_validity} as a conditional expectation of the score times a correction factor, as discussed in the following proposition.
\begin{restatable}[]{proposition}{Correction}\label{prop:correction}
In the setting of Section~\ref{sec:conditions_for_valid_observational_randomized_comparisons}, under Assumptions~\ref{asmp:internal_validity} and~\ref{asmp:external_validity}, the conditional mean potential outcome in the target distribution is identified as 
\begin{equation}
    \E[Y_a \mid S = 0, G = i] = \frac{\P(S = 0)}{\P(S = 0 \mid G = i)}\E[Y^a(\eta_0, \pi_0) \mid G = i],
\end{equation}
where $Y^a(\eta, \pi)$ is defined as in Equation~\ref{eq:definition_transport_score_proof}.
\begin{equation}\label{eq:definition_transport_score_proof}
Y^{a}(\eta, \pi) \coloneqq \frac{1}{\pi} \left(\1{S = 1, A = a} \cdot \frac{1 - p(X)}{p(X) e_a(X)} \cdot  \{ Y - g_a(X)\} + (1 - S) g_a(X)\right)
\end{equation}
where $\eta \coloneqq (g_a, e_a, p)$ with true underlying parameters $\eta_0 = (g_{a0}, e_{a0}, p_0)$, $\pi \coloneqq \P(S = 0)$ with true value $\pi_0$, $g_a(X) \coloneqq \E[Y \mid A = a, S = 1, X]$, $p(X) \coloneqq \P(S = 1 \mid X)$, and $e_a(X) \coloneqq \P(A = a \mid S = 1, X]$. 
\end{restatable}
A proof is provided in Appendix~\ref{sec:proofs}. Note that this is equivalent to replacing the estimate of $1 / \P(S = 0)$ in the score with an estimate of $1 / \P(S = 0 \mid G = i)$, before computing the empirical conditional expectations of the score.

Now, for each observational dataset, we construct estimates $\tauhki, \sigmahksi$ for $i \in \cI_{R}$ as follows:
\begin{enumerate}
\item We collect observational samples from the two validation groups \{lbw, married\} and \{hbw, married\}, which we denote as $G = 0, G = 1$ respectively.  We combine these observational samples with the samples from the RCT, using $S = 0$ to denote RCT samples (the target distribution) and $S = 1$ to denote observational samples.
\item We define our signal for each sample as 
\begin{equation}
    \label{eq:tildeY}
    \tilde{Y}_i(\hat{\eta}, \hat{\pi}_g) \coloneqq \tilde{Y}^1_i(\hat{\eta}, \hat{\pi}_g) - \tilde{Y}^0_i(\hat{\eta}, \hat{\pi}_g)
\end{equation}
where we define the modified score $\tilde{Y}^a_i$, in light of Proposition~\ref{prop:correction}, as 
\begin{equation}
\tilde{Y}^{a}_i(\hat{\eta}, \hat{\pi}_g) \coloneqq \frac{1}{\hat{\pi}_g(G_i)} \left(\1{S_i = 1, A_i = a} \cdot \frac{1 - \hat{p}(X_i)}{\hat{p}(X_i) \hat{e}_a(X_i)} \cdot  \{ Y_i - \hat{g}_a(X_i)\} + (1 - S_i) \hat{g}_a(X_i)\right)\label{eq:definition_transport_score_revised},
\end{equation}
where $\hat{\pi}_g(G_i)$ is defined as an estimate of $\pi_{g}(G_i) := \P(S = 0 \mid G_i)$, computed using empirical averages.
\item We use 3-fold cross-fitting as described in \citet{Semenova2021-zf} to generate the signals for each sample, such that for the $i$-th datapoint, the score $\tilde{Y}_i(\hat{\eta}, \hat{\pi})$ uses plug-in estimates $\hat{\eta} = (\hat{g}_1, \hat{g}_0, \hat{e}_1, \hat{p})$ that are learned on the folds that do not include the $i$-th datapoint, and $\hat{\pi}$ is estimated using empirical averages. In practice, we use a multi-layer perceptron (MLP) regressor for estimating $\hat{g}_a$, and $\ell_2$-regularized logistic regression for estimating $\hat{e}_1, \hat{p}$, with hyperparameters described in Section~\ref{sec:exp_details}. For each model, we reserve 20\% of the current fold in the cross fitting procedure as a validation set to do hyperparameter selection. 
\item Finally, we estimate $\tauhki$ as the empirical average $\E[\tilde{Y}(\hat{\eta}, \hat{\pi}_g) \mid G = i]$, and we use the empirical conditional variance of this score to estimate the variance $\sigmahksi$.
\end{enumerate}

We construct the RCT estimate $\tauhzi$ (using the RCT sample alone) as the difference of the empirical conditional means $\E_{N_0}[Y \left(\frac{\1{A = 1}}{\hat{P}(A = 1)} - \frac{\1{A = 0}}{1 - \hat{P}(A = 1)}\right) \mid G = i]$, where $\hat{P}(A = 1)$ is an empirical average.  We compute $\sigmahzsi$ as the empirical conditional variance of this quantity.  We then conduct testing, as described in Algorithm~\ref{alg:high_level}.

\paragraph{Asymptotic normality of transported estimators}

We herein provide sufficient assumptions that guarantee the asymptotic normality of our transported GATE estimators, i.e. the empirical average $\E[\tilde{Y}(\hat{\eta}, \hat{\pi}_g) | G=i]$:

\begin{restatable}[Observational dataset covers the whole support of covariates]{assumption}{BoundedP}
    \label{asmp:bounded_p}
    \[\inf_{x \in \mathcal{X}} p_0(x) = \varepsilon_p > 0\]
\end{restatable}

Note that Assumption~\ref{asmp:bounded_p} is implied by Assumption~\ref{asmp:support}.

\begin{restatable}[Bounded within-subgroup variance of conditional treatment effects in the RCT]{assumption}{BoundedCateVariance}
    \label{asmp:bounded_cate_variance}
    \[\sup_{x \in \mathcal{X}} var[g_{10}(x)-g_{00}(x)|G=i, S=0] = \sigma_{\tau i}^2 < \infty\]
\end{restatable}

\begin{restatable}[Overlap between treatments in the observational dataset]{assumption}{BoundedE}
    \label{asmp:bounded_e}
    \[\inf_{x \in \mathcal{X}} \min(e_{00}(x), e_{10}(x)) = \varepsilon_e > 0\]
\end{restatable}

\begin{restatable}[Finite outcome conditional variance in the observational dataset]{assumption}{BoundedVariance}
    \label{asmp:bounded_variance}
    \[\max_{a \in \{0,1\}} \sup_{x \in \mathcal{X}} \E[(Y-g_{a0}(x))^2|X=x, S=1, A=a] = \bar{\sigma}^2 < \infty\]
\end{restatable}

\begin{restatable}[Properties of the nuisance function estimators]{assumption}{Nuisance}
    \label{asmp:nuisance} Let $\hat{\eta}_{(n)}$ be a sequence of estimators for $\eta$ indexed by the size of the cross-fitting training fold $n$. We assume that there exists
    \begin{itemize}
        \item $\epsilon_n = o_P(1)$, a sequence of positive numbers
        \item $\mathcal{T}_n$, a sequence of nuisance function vector sets in the neighborhood of $\eta_0 = (g_{10}, g_{00}, e_{10}, p_0)$ satisfying $\P(\hat{\eta}_{(n)}\in\mathcal{T}_n) \ge 1-\epsilon_n$
        \item $\mathbf{g}_n, \mathbf{e}_n, \mathbf{p}_n$, sequences of worst root mean square errors for the nuisance functions $g_1, g_0, e_1, p$, defined as follows:
        \begin{align*}
            \mathbf{g}_n &\coloneqq \max_{a \in \{0,1\}}\sup_{\eta \in \mathcal{T}_n} \sqrt{\E[g_a(X)-g_{a0}(X)]^2}\\
            \mathbf{e}_n &\coloneqq \sup_{\eta \in \mathcal{T}_n} \sqrt{\E[e_1(X)-e_{10}(X)]^2}\\
            \mathbf{p}_n &\coloneqq \sup_{\eta \in \mathcal{T}_n} \sqrt{\E[p(X)-p_0(X)]^2}
        \end{align*}
    \end{itemize}
so that the following assumptions hold:
    \begin{enumerate}[label={\textbf{Assumption}~\textbf{\Alph*:}}, ref={\theassumption.\Alph*}, align = left, leftmargin = 1em]
        \item \label{asmp:rate}(Rate of nuisance error)
        \[\mathbf{g}_n \vee \mathbf{e}_n \vee \mathbf{p}_n = o_P(1)\]
        \item \label{asmp:rate_product}(Rate of nuisance error product) 
        \[\sqrt{n} \mathbf{g}_n (\mathbf{e}_n \vee \mathbf{p}_n) = o_P(1)\]
        \item \label{asmp:bounded_nuisance}(Bounded nuisance estimates) \[\sup_{\eta \in \cup_{n=1}^{\infty} \mathcal{T}_n}\left( \max_{a\in\{0,1\}} \sup_{x\in\mathcal{X}} \left|g_a(x)\right| \vee \sup_{x\in\mathcal{X}} \left|\frac{1}{p(x)}\right| \vee \sup_{x\in\mathcal{X}} \left|\frac{1}{e_1(x)}\right| \vee \sup_{x\in\mathcal{X}} \left|\frac{1}{e_0(x)}\right|\right) = \bar{\mathcal{C}} < \infty\]
    \end{enumerate}
\end{restatable}

\begin{restatable}{theorem}{Normal}
    \label{prop:normal}
    Suppose Assumptions~\ref{asmp:bounded_p} to~\ref{asmp:nuisance} hold. Then, the empirical average, $\E[\tilde{Y}(\hat{\eta}, \hat{\pi}_g) | G=i]$, where $\tilde{Y}$ is defined in Equation~\ref{eq:tildeY} and $\hat{\eta}$ is estimated with cross-fitting, is asymptotically normal.
\end{restatable}

\textit{Remark}: As we will prove later in section~\ref{subsec:normal}, Assumptions~\ref{asmp:bounded_p} to~\ref{asmp:bounded_variance} guarantee that when the nuisance function vector $(g_{10}, g_{00}, e_{10}, p_0)^\top$ is known (i.e. need not be estimated), the transported GATE estimator is asymptotically normal. In practice, $(g_{10}, g_{00}, e_{10}, p_0)^\top$ is not known and has to be estimated, so Assumption~\ref{asmp:nuisance} lays out sufficient properties the nuisance function vector estimator needs to satisfy. In particular, Assumptions~\ref{asmp:rate} and~\ref{asmp:rate_product} permit that the convergence rate of estimators can be slower than $o_P(n^{-1/2})$, which is useful when $X$ is high-dimensional and machine learning models are required to estimate the nuisance functions. To date, a variety of commonly-used machine learning models have been shown to enjoy a convergence rate of at least $o_P(n^{-1/4})$, e.g. \cite{buhlmann2011statistics, belloni2011square, belloni2011inference} for certain $\ell_1$ penalized models, \cite{wager2015adaptive} for a class of regression trees and random forests, and \cite{chen1999improved} for a class of neural nets. This implies when these models are applied to the estimation of $(g_{10}, g_{00}, e_{10}, p_0)^\top$, Assumptions~\ref{asmp:rate} and~\ref{asmp:rate_product} hold, so our transported GATE estimator is asymptotically normal and Assumption~\ref{asmp:all_asymptotic_normal} is satisfied.

\paragraph{Constructing Confidence Intervals for the Extrapolated Effects} In our experimental setup, the data generating distribution for all observational studies is identical, so no transportation of effects is required, which enables the application of existing results.  We use the doubly-robust score \citep{robins1994estimation, Robins1995-cn} as the signal for the conditional average treatment effect,
\begin{equation}
    Y(\eta) = \mu(1,X) - \mu(0,X) + \frac{A(Y-\mu(1,X))}{s(X)} - \frac{(1-A)(Y-\mu(0,X))}{1-s(X)},\label{eq:doubly_robust_score}
\end{equation}
where $\eta \coloneqq (\mu, s)$, and $\mu(A, X)  \coloneqq \E[Y|A,X]$, and $s(X) \coloneqq \P(A = 1 \mid X)$.  
We use a multi-layer perceptron (MLP) regressor as a plug-in estimate $\hat{\mu}$ of $\mu$, and $\ell_2$-regularized logistic regression as a plug-in estimate $\hat{s}$ of $s$, with hyperparameters described in Section~\ref{sec:exp_details}.  

Following example 2.2 from \citet{Semenova2021-zf}, we approximate the conditional treatment effect with a linear combination of subgroup dummy variables $G = (G_0, G_1, G_2, G_3)^\top$, so the combination weights correspond to the GATEs $\tauk = (\tauk_0, \tauk_1, \tauk_2, \tauk_3)$. This amounts to regressing the estimated signal $\hat{Y}_i(\hat{\eta})$ with $G$. As long as the propensity score is bounded above and below away from 0 and 1 (Assumption 4.10(a) of \citet{Semenova2021-zf}), and the convergence rates of the response surface and propensity score estimates are sufficiently fast (Assumption 4.11), Corollary 4.1 and a set of mild technical conditions justify Theorem 3.1 in \citet{Semenova2021-zf}, which gives a result on pointwise asymptotic normality for the regression coeffcients $\tauhk = (\tauhk_0, \tauhk_1, \tauhk_2, \tauhk_3) \in \R^{4}$, so that for any unit vector $\gamma \in \R^4$ where $\norm{\gamma} = 1$,
\[\lim_{N_k \rightarrow \infty}\sup_{t \in \R} \left| \P \left( \frac{\sqrt{N_k}\gamma^\top(\tauhk -\tauk)}{\sqrt{\gamma^\top \Omega \gamma}} < t\right)-\Phi(t) \right| = 0\]
where $\Omega$ can be consistently estimated with Equation 2.5 in \citet{Semenova2021-zf}
\[\hat{\Omega} = \left(\frac{1}{N_k}\sum_j G_jG_j^\top\right)^{-1} \left(\frac{1}{N_k}\sum_j G_jG_j^\top(\hat{Y}_j(\hat{\eta})-G_j^\top\tauhk)^2\right)\left(\frac{1}{N_k}\sum_j G_jG_j^\top\right)^{-1}\]
Setting $\gamma$ as $1$ in the $(i+1)$th element and $0$ elsewhere thus yields 
\[\lim_{N_k \rightarrow \infty}\sup_{t \in \R} \left| \P \left( \frac{\sqrt{N_k} (\tauhki-\tauki)}{\sqrt{\Omega_{ii}}} < t\right)-\Phi(t) \right| = 0\]
We therefore estimate $\sigmahksi$, the variance of $\tauhki$, with $\hat{\Omega}_{ii}$, and as this converges in probability to $\Omega_{ii}$, the asymptotic normality of the above follows via Slutsky's theorem.

\section{Proofs}%
\label{sec:proofs}
\subsection{Proofs for propositions and theorems}
\label{subsec:proofs}
\ValidityOfNormalTest*
\begin{proof}
  As $N \rightarrow \infty$, we have it that 
  \begin{align*}
    \sqrt{\rho N} (\tauhki - \tauki) &\cid \cN(0, \sigmaksi) \\
    \sqrt{(1 - \rho) N} (\tauhzi - \taui) &\cid \cN(0, \sigmazsi)
  \end{align*}
  where we have written $\rho N$ in place of $N_k$, and similarly for $N_0$.  By Slutsky's theorem, we can multiply by the constants $\rho^{-1/2}$ and ${(1 - \rho)}^{-1/2}$ to get both results in terms of $\sqrt{N}$.  We can then use independence of $\tauhk, \tauhz$ to write that
  \begin{equation*}
    \sqrt{N} \underbrace{\begin{pmatrix}
      \tauhki - \tauki \\
      \tauhzi - \taui
    \end{pmatrix}}_{Z - \theta} \cid \cN\left(\begin{pmatrix} 0 \\ 0 \end{pmatrix}, \begin{bmatrix}
      \sigmaksi / \rho & 0 \\
      0 & \sigmazsi / (1 - \rho)
    \end{bmatrix}\right).
  \end{equation*}
  We now apply the Delta method. Let $Z = (\tauhki, \tauhzi)$ denote the (column) vector of estimates, and similarly let $\theta = (\tauki, \tau_i)$. Letting $f(X) = X_1 - X_2$, we can argue that 
  \begin{equation*}
    \sqrt{N} (Z - \theta) \cid \cN(0, \Sigma) \implies \sqrt{N} (f(Z) - f(\theta)) \cid \cN\left(0, \nabla f{(\theta)}^\top \Sigma \ \nabla f(\theta) \right),
  \end{equation*}
  where the resulting variance is given by 
  \begin{equation*}
   \nabla f{(\theta)}^\top \Sigma \ \nabla f(\theta) = \frac{\sigmaksi}{\rho} + \frac{\sigmazsi}{1 - \rho},
  \end{equation*}
  and $f(Z) - f(\theta) = \tauki - \tau_i - \muki$.  
  \begin{align*}
    \sqrt{N} (\tauki - \tau_i - \muki) &\cid \cN\left(0, \frac{\sigmaksi}{\rho} + \frac{\sigmazsi}{1 - \rho}\right),
  \end{align*}
  and accordingly that 
  \begin{align*}
  \frac{\tauk_i - \tau_i - \muki}{\sqrt{\frac{\sigmaksi}{N_k} + \frac{\sigmazsi}{N_0}}} & \cid \cN(0, 1),
  \end{align*}
  where this also holds (by Slutsky's theorem) with $\sigmaks_i$ and $\sigmazsi$ replaced by their empirical estimates, which converge in probability.
\end{proof}

\Properties*
\begin{proof}
  \textbf{(1)} By asymptotic normality and consistency of each dimension of $\tauk$, the test statistic $\Thnki$ converges in distribution to $\cN(0, 1)$.  As a result, for each $i \in \cI_{R}$, the probability that $\abs{\Thnki} > z_{\alpha / (4\abs{\cI_{R}})}$ converges to $\alpha / (2\abs{\cI_{R}})$.  By an application of the union bound, the probability that this occurs for any $i \in \cI_{R}$ is bounded by $\alpha / 2$. Similarly, by the assumed properties of $\tauk$, the probability that the confidence interval $[\Lki(\alpha/2), \Uki(\alpha/2)]$ fails to capture the true value of $\tau_i$ converges to $\alpha/2$.  By another application of the union bound, for each $i \in \cI_{O}$, the probability that either $\tauk$ is not selected or $\tau_i$ is not contained in the interval is upper bounded by $\alpha$.  The result follows.

  \textbf{(2)} By asymptotic normality of each $\tauk$, the power calculation in Equation~\eqref{eq:definition_beta} holds, and as $N \rightarrow \infty$, the probability of rejecting the null hypothesis converges to zero as $\sigma^2_{k, 0}$ becomes arbitrarily large, which occurs as both $N_k, N_0 \rightarrow \infty$.
\end{proof}

\Correction*
\begin{proof}
First, we can observe by standard arguments that the conditional expectation of $Y^a(\eta_0, \pi_0)$ given $X$ is given by the following
\begin{equation*}
\E[Y^a(\eta_0, \pi_0) \mid X = x] = \E\left[\frac{1 - S}{\P(S = 0)} g_{a0}(X) \middle| X = x\right],
\end{equation*}
because the first term in Equation~\eqref{eq:definition_transport_score_proof} is mean-zero conditioned on $X = x$. This follows by the law of total expectation: for any event where $S = 1, A = a$ does not hold, the first term is zero due to the indicator, and for any other event $S = 1, A = a, X = x$, the first term is mean-zero, since the first term becomes a constant (determined by $S = 1, A = a, X = x$) times a mean-zero random variable $Y - \E[Y \mid A = a, S = 1, X = x]$.  

As a result, we can write that 
\begingroup
\allowdisplaybreaks
\begin{align*}
    &\E[Y^a(\eta_0, \pi_0) \mid G = i] \\
    &= \E\left[\frac{1 - S}{\P(S = 0)} \E[Y \mid A = a, S = 1, X] \middle| G = i\right] \\
    &= \frac{1}{\P(S = 0)} \int_{x} \sum_{s} \1{s = 0} \E[Y \mid A = a, S = 1, x] p(s, x \mid G = i) dx \\
    &= \frac{1}{\P(S = 0)} \int_{x} \sum_{s} \1{s = 0} \E[Y_a \mid S = 1, x] p(s, x \mid G = i) dx & \text{Assumption~\ref{asmp:internal_validity}}, Y_a \indep A \mid X, S = 1 \\
    &= \frac{1}{\P(S = 0)} \int_{x} \sum_{s} \1{s = 0} \E[Y_a \mid S = 0, x] p(s, x \mid G = i) dx& \text{Assumption~\ref{asmp:external_validity}}, Y_a \indep S \mid X \\
    &= \frac{1}{\P(S = 0)} \int_{x} \E[Y_a \mid x, S = 0] p(S = 0, x \mid G = i) dx \\
    &= \frac{1}{\P(S = 0)} \int_{x} \E[Y_a \mid x, S = 0] p(x \mid S = 0, G = i) \P(S = 0 \mid G = i) dx \\
    &= \frac{\P(S = 0 \mid G = i)}{\P(S = 0)} \int_{x} \E[Y_a \mid x, S = 0] p(x \mid S = 0, G = i) dx \\
    &= \frac{\P(S = 0 \mid G = i)}{\P(S = 0)} \int_{x} \E[Y_a \mid x, S = 0, G = i] p(x \mid S = 0, G = i) dx &X = x \Rightarrow G = i, \forall x: p(x \mid G = i) > 0 \\
    &= \frac{\P(S = 0 \mid G = i)}{\P(S = 0)} \E[Y_a \mid S = 0, G = i] 
\end{align*}
\endgroup
and the result follows from dividing both sides by the first term on the right-hand side, which we can observe is equivalent to multiplying both sides by 
\begin{equation}
  \frac{\P(S = 0)}{\P(S = 0 \mid G = i)} = \frac{\P(S = 0)\P(G = i)}{\P(S = 0, G = i)} = \frac{\P(G = i)}{\P(G = i \mid S = 0)}
\end{equation}
\end{proof}

\subsection{Asymptotic normality of cross-fitted transported GATE estimators}
\label{subsec:normal}
\Normal*

\textit{Proof sketch}: Our strategy for the proof consists of two stages. First, we show that if the nuisance function is known to be $\eta_0$ and plugged into the estimator as $\E[\tilde{Y}(\eta_0,\hat{\pi}_g)|G=i]$, the resulting estimator, which we later refer to as the oracle estimator, is asymptotically normal. Second, we show that even if the true nuisance function is not known, as long as we have an estimator, $\hat{\eta}$, of the nuisance function that follows certain properties, the resulting estimator $\E[\tilde{Y}(\hat{\eta},\hat{\pi}_g)|G=i]$ converges to the oracle estimator in probability. Then, by Slutsky's Theorem, the resulting estimator is also asymptotically normal.

Before diving into the first stage of the proof, we introduce additional notation to reflect the cross-fitting nature of our GATE estimator.  Let the combined sample size of the observational study and RCT be $N$ with sample indices $[N] \coloneqq \{1,2,...,N\}$. We denote $(I_m)_{m=1}^M$ as a $M$-fold random partition of $[N]$, so that each fold has size $N_M = N/M$. The plug-in nuisance function estimate for the $m^{\text{th}}$ fold, $\hat{\eta}_m$, is then estimated from the rest of the folds $I_m^c \coloneqq [N]\backslash I_m$. For brevity, we denote the size of the rest of the folds as $N_M^c = N - N/M$.

We now restate the definition of the treatment effect signal $\tilde{Y}_j(\eta, \pi_g) = \tilde{Y}_j((g_1, g_0, e_1, p)^\top, \pi_g)$:
\begin{align*}
    \tilde{Y}_j(\eta, \pi_g) &\coloneqq \tilde{Y}^1_j(\eta, \pi_g) - \tilde{Y}^0_j(\eta, \pi_g)\\
    \tilde{Y}^{a}_j(\eta, \pi_g) &\coloneqq \frac{1}{\pi_g(G_i)} \left(\1{S_j = 1, A_j = a} \cdot \frac{1 - p(X_j)}{p(X_j) e_a(X_j)} \cdot  \{ Y_j - g_a(X_j)\} + (1 - S_j) g_a(X_j)\right)
\end{align*}
In the remainder of the development, we will drop the subscript $j$, which represents one of the $N$ samples, for conciseness. 

\textbf{Stage 1} — \textit{Proving the asymptotic normality of the oracle estimator}

For brevity, we define the following unweighted signal:
\begin{definition}[Unweighted signal functional]
\begin{align*}
    \mathcal{Y}(\eta) &= \pi_g(G)\tilde{Y}(\eta, \pi_g)\\
    &= \pi_g(G)(\tilde{Y}^1(\eta, \pi_g)-\tilde{Y}^0(\eta, \pi_g))\\ 
    &= (1-S)(g_1(X)-g_0(X)) + S\frac{1-p(X)}{p(X)}\frac{(A-e_1(X))(Y-g_{A}(X))}{e_1(X)e_0(X)}
\end{align*}
\end{definition}

From the proof of Proposition ~\ref{prop:correction}, we have the following identities for the unweighted signals:

\begin{lemma}[Conditional mean of unweighted (oracle) signal]
    \label{lem:oracle_signal_mean}
\textit{The conditional mean of the unweighted (oracle) signal is equivalent to the following:}
    \begin{align*}
        &\E[\mathcal{Y}(\eta_0)|G = i] = \tau_i\pi_g(i)\\
        &\E[\mathcal{Y}(\eta_0)|G = i, S = 0] = \tau_i
\end{align*}.
\end{lemma}

\begin{proof}
    First, we have, 
    \begin{align*}
        \E[\mathcal{Y}(\eta_0) | G=i] &= \E[\pi_g (G) \tilde{Y}(\eta_0, \pi_g) | G = i] \\ 
        &= \E[\pi_g(i) \tilde{Y}(\eta_0, \pi_g) | G = i] \\
        &= \pi_g(i) \E[\tilde{Y}(\eta_0, \pi_g) | G = i] \\
        &= \tau_i \pi_g(i)  
    \end{align*} 
    
    Next, using Definition D.1 of the unweighted signal functional and the fact that we condition on $S=0$, we have, 
    \begin{align*}
        \E[\mathcal{Y}(\eta_0)|G=i,S=0] = \E\left[g_{10}(X)-g_{00}(X)|G=i, S=0\right],
    \end{align*}
    which is $\tau_i$ as desired.
\end{proof}

In addition, we can rewrite our estimator $\hat{\tau}_i \coloneqq \E[\tilde{Y}(\hat{\eta},\hat{\pi}_g)|G=i]$ with the unweighted signals:
\begin{align*}
    \hat{\tau}_i =\E[\tilde{Y}(\hat{\eta},\hat{\pi}_g)|G=i] &= \frac{\sum_m \sum_{j \in I_m} \mathbf{1}(G_j = i)\tilde{Y}(\hat{\eta}_m, \hat{\pi}_g)}{\sum_j\mathbf{1}(G_j = i)}\\
    &= \frac{\sum_m \sum_{j \in I_m} \mathbf{1}(G_j = i)\frac{1}{\hat{\pi}_g(G_j)}\mathcal{Y}(\hat{\eta}_m)}{\sum_j\mathbf{1}(G_j = i)}, \quad\text{from \textbf{Def. D.1}}\\
    &= \frac{\sum_m \sum_{j \in I_m} \mathbf{1}(G_j = i)\frac{1}{\hat{\pi}_g(i)}\mathcal{Y}(\hat{\eta}_m)}{\sum_j\mathbf{1}(G_j = i)}\\
    &= \frac{1}{\hat{\pi}_g(i)}\frac{\sum_m \sum_{j \in I_m} \mathbf{1}(G_j = i)\mathcal{Y}(\hat{\eta}_m)}{\sum_j\mathbf{1}(G_j = i)}\\
    &= \frac{\sum_m \sum_{j \in I_m} \mathbf{1}(G_j = i)\mathcal{Y}(\hat{\eta}_m)}{\frac{\sum_j\mathbf{1}(G_j = 1, S_j = 0)}{\sum_j\mathbf{1}(G_j = i)}\sum_j\mathbf{1}(G_j = i)}\\
    &= \frac{\sum_m \sum_{j \in I_m} \mathbf{1}(G_j = i)\mathcal{Y}(\hat{\eta}_m)}{\sum_j \mathbf{1}(G_j = 1, S_j = 0)}
\end{align*}

Now, using the above expression, we can define the oracle estimator, where we \textit{know} the true value of $\eta$, which is $\eta_0$:
\begin{definition}[Oracle GATE Estimator]
    \[\hat{\tau}_{i0} \coloneqq \frac{\sum_j \mathbf{1}(G_j = i)\mathcal{Y}_j(\eta_0)}{\sum_j\mathbf{1}(G_j = i, S_j = 0)}\]
\end{definition}

To show the asymptotic distribution of the oracle GATE estimator, we restate several assumptions:

\BoundedP*
\BoundedCateVariance*
\BoundedE*
\BoundedVariance*

These assumptions ensure that the oracle signals have finite conditional variance, which we prove in the following lemma.

\begin{restatable}[Finite conditional variance of unweighted oracle signal]{lemma}{OracleVariance}
    \label{lem:oracle_variance}
    Under Assumptions \ref{asmp:bounded_p} - \ref{asmp:bounded_variance}, we have that,
    \[var[\mathcal{Y}(\eta_0)|G = i] \coloneqq \sigma_i^2 < \infty, \forall i \in [d]\]
\end{restatable}

\begin{proof}
    \begin{align*}
        &var[\mathcal{Y}(\eta_0)|G = i]&\\ =&\E[\mathcal{Y}^2(\eta_0)|G = i] - [\E[\mathcal{Y}(\eta_0)|G = i]]^2&\\
        =&\E\left[\left((1-S)(g_{10}(X)-g_{00}(X))+\phantom{\frac{1}{1}}\right.\right.&\text{Lem. }\ref{lem:oracle_signal_mean}\\
        &\qquad \left.\left.S\frac{1-p_0(X)}{p_0(X)}\frac{(A-e_{10}(X))(Y-g_{A0}(X))}{e_{10}(X)e_{00}(X)}\right)^2\middle|G = i\right] - \pi_g(i)^2\tau_i^2&\\
        =&\left\{\E\left[(1-S)(g_{10}(X)-g_{00}(X))^2\middle|G=i\right]+\phantom{\frac{1}{1}} \right.& \text{since }S \in \{0, 1\}\\
        &\qquad \left.\E\left[S\left(\frac{1-p_0(X)}{p_0(X)}\frac{(A-e_{10}(X))(Y-g_{A0}(X))}{e_{10}(X)e_{00}(X)}\right)^2\middle|G = i\right]\right\} - \pi_g(i)^2\tau_i^2&\\
        =&\left\{\E\left[(g_{10}(X)-g_{00}(X))^2\middle|G=i, S=0\right]\pi_g(i)+\phantom{\frac{1}{1}} \right.&\\
        &\qquad \left.\E\left[\left(\frac{1-p_0(X)}{p_0(X)}\frac{(A-e_{10}(X))(Y-g_{A0}(X))}{e_{10}(X)e_{00}(X)}\right)^2\middle|G = i, S = 1\right](1-\pi_g(i))\right\} - \pi_g(i)^2\tau_i^2&\\
        =&\left\{\left[var\left[g_{10}(X)-g_{00}(X)\middle|G=i, S=0\right]+\tau_i^2\right]\pi_g(i)+\phantom{\frac{1}{1}} \right.&\\
        &\qquad \left.\E\left[\left(\frac{1-p_0(X)}{p_0(X)}\frac{(A-e_{10}(X))(Y-g_{A0}(X))}{e_{10}(X)e_{00}(X)}\right)^2\middle|G = i, S = 1\right](1-\pi_g(i))\right\} - \pi_g(i)^2\tau_i^2&\\
        <&\left\{\left[\sigma_{\tau i}^2+\tau_i^2\right]\pi_g(i)+\frac{\E[(Y-g_{A0}(X))^2|G=i, S=1]}{\varepsilon_\pi^2 \varepsilon_e^2 (1-\varepsilon_e)^2}(1-\pi_g(i))\right\}-\pi_g(i)^2\tau_i^2 &\text{Asmp. }\ref{asmp:bounded_p} - \ref{asmp:bounded_e}\\
        \le&\left\{\left[\sigma_{\tau i}^2+\tau_i^2\right]\pi_g(i)+\frac{\bar{\sigma}^2}{\varepsilon_\pi^2 \varepsilon_e^2 (1-\varepsilon_e)^2}(1-\pi_g(i))\right\}-\pi_g(i)^2\tau_i^2 < \infty &\text{Asmp. }\ref{asmp:bounded_variance}
    \end{align*}
\end{proof}

Now, using the above lemmas, we are ready to prove the main result of stage 1 of the proof, stated below.
\begin{restatable}[Asymptotic normality of oracle GATE estimator]{proposition}{OracleNormality} \label{prop:oracle_normality}
    Under Assumptions \ref{asmp:bounded_p} - \ref{asmp:bounded_variance},
    \[\sqrt{N}(\hat{\tau}_{i0} - \tau_i) \xrightarrow{d} \mathcal{N}\left(0, \frac{\sigma_{i}^2 - \pi_g(i)(1-\pi_g(i))\tau_i^2}{\pi_g(i)^2\P(G = i)}\right)\]
\end{restatable}

\begin{proof}
    We have that,
    \begin{align*}
        &\sqrt{N}(\hat{\tau}_{i0}-\tau_i)\\
        = &\sqrt{N}\left(\frac{\sum_j \mathbf{1}(G_j = i)\mathcal{Y}_j(\eta_0)}{\sum_j\mathbf{1}(G_j = i, S_j = 0)}-\tau_i\right) &\\
        = &\sqrt{N}\left(\frac{\sum_j \mathbf{1}(G_j = i)\mathcal{Y}_j(\eta_0)}{\sum_j\mathbf{1}(G_j = i, S_j = 0)}-\frac{\sum_j\mathbf{1}(G_j = i, S_j = 0)\tau_i}{\sum_j\mathbf{1}(G_j = i, S_j = 0)}\right) &\\
        = &\sqrt{N}\frac{\sum_j \mathbf{1}(G_j = i)(\mathcal{Y}_j(\eta_0)-\tau_i\mathbf{1}(S_j=0))}{\sum_j\mathbf{1}(G_j = i, S_j =0)}&\\
        = &\frac{\sqrt{N}\frac{1}{N}\sum_j \mathbf{1}(G_j = i)(\mathcal{Y}_j(\eta_0)-\tau_i\mathbf{1}(S_j=0))}{\frac{1}{N}\sum_j\mathbf{1}(G_j = i, S_j = 0)}\;\; \genfrac{}{}{0pt}{}{\xrightarrow{d} \mathcal{N}(0, (\sigma^2_i- \tau_i^2\pi_g(i)(1-\pi_g(i)))\P(G=i))\hfill}{\xrightarrow{p}\P(G=i,S=0) = \P(G=i)\pi_g(i)\hfill}&\genfrac{}{}{0pt}{}{\hfill \text{Proven below}}{\hfill\text{From WLLN}}\\
        \xrightarrow{d}&\mathcal{N}\left(0, \frac{\sigma^2_i- \tau_i^2\pi_g(i)(1-\pi_g(i))}{\pi_g(i)^2\P(G=i)}\right)&\text{Slutsky's lemma}
    \end{align*}
    In the above, we used the fact that,
    \begin{equation*}
        \sqrt{N}\left[\frac{1}{N}\sum_j \mathbf{1}(G=i)(\mathcal{Y}(\eta_0) - \tau_i\mathbf{1}(S=0)) \right] \xrightarrow{d} \mathcal{N}(0, (\sigma^2_i- \tau_i^2\pi_g(i)(1-\pi_g(i)))\P(G=i))
    \end{equation*}
    To show this, we observe that 
    \begin{align*}
        &\E[\mathbf{1}(G=i)(\mathcal{Y}(\eta_0) - \tau_i\mathbf{1}(S=0))]&\\ 
        =&\E[\mathcal{Y}(\eta_0) - \tau_i\mathbf{1}(S=0)|G=i]\P(G=i)&\\
        =&(\E[\mathcal{Y}(\eta_0)|G=i] - \tau_i\E[\mathbf{1}(S=0)|G=i])\P(G=i)&\\
        =&(\E[\mathcal{Y}(\eta_0)|G=i] - \tau_i\pi_g(i))\P(G=i) = 0&\text{Lem. }\ref{lem:oracle_signal_mean}\\\\
        &var[\mathbf{1}(G=i)(\mathcal{Y}(\eta_0) - \tau_i\mathbf{1}(S=0))]&\\
        =&\E[\mathbf{1}(G=i)(\mathcal{Y}(\eta_0) - \tau_i\mathbf{1}(S=0))]^2 - (\E[\mathbf{1}(G=i)(\mathcal{Y}(\eta_0) - \tau_i\mathbf{1}(S=0))])^2&\\
        =&\E[\mathbf{1}(G=i)(\mathcal{Y}(\eta_0) - \tau_i\mathbf{1}(S=0))]^2&\\
        =&\E[\mathbf{1}(G=i)(\mathcal{Y}(\eta_0) - \tau_i\mathbf{1}(S=0))^2] & \mathbf{1}(G=i) \in \{0,1\}\\
        =&\E[(\mathcal{Y}(\eta_0) - \tau_i\mathbf{1}(S=0))^2|G=i]\P(G=i)&\\
        =&\left\{\E[\mathcal{Y}^2(\eta_0)|G=i] + \tau_i^2\E[\mathbf{1}(S=0)|G=i] -2\tau_i\E[\mathcal{Y}(\eta_0)\mathbf{1}(S=0)|G=i]\right\}\P(G=i)&\mathbf{1}(S=0) \in \{0,1\}\\
        =&\left\{var[\mathcal{Y}(\eta_0)|G=i] + (\E[\mathcal{Y}(\eta_0)|G=i])^2 + \tau_i^2\pi_g(i)\right.&\\ 
        &\qquad\left. -2\tau_i\pi_g(i)\E[\mathcal{Y}(\eta_0)|G=i, S=0]\right\}\P(G=i)&\\
        =&\left\{\sigma_i^2 + \tau_i^2 \pi_g(i)^2 + \tau_i^2\pi_g(i) -2\tau_i^2\pi_g(i)\right\}\P(G=i)&\text{Asmp. }\ref{lem:oracle_variance}\text{, Lem. }\ref{lem:oracle_signal_mean}\\
        =&(\sigma_i^2 - \tau_i^2\pi_g(i)(1-\pi_g(i)))\P(G=i) < \infty&
    \end{align*}
    Therefore, from central limit theorem,
    \begin{equation*}
        \sqrt{N}\left[\frac{1}{N}\sum_j \mathbf{1}(G=i)(\mathcal{Y}(\eta_0) - \tau_i\mathbf{1}(S=0)) \right] \xrightarrow{d} \mathcal{N}(0, (\sigma^2_i- \tau_i^2\pi_g(i)(1-\pi_g(i)))\P(G=i))
    \end{equation*}
\end{proof}

\textbf{Stage 2} — \textit{Proving the asymptotic normality of the cross-fitted estimator, $\tilde{Y}(\hat{\eta},\hat{\pi}_g)$}

With asymptotic normality of the oracle estimator shown above in Stage 1, we can show the asymptotic normality of the cross-fitted estimator (i.e. our estimator) by decomposing its error into the error of the oracle estimator and the difference between our estimator and the oracle estimator:
\begin{align*}
    \sqrt{N}(\hat{\tau}_i - \tau_i) &=
    \sqrt{N}(\hat{\tau}_{i0} - \tau_i) + \sqrt{N}(\hat{\tau}_i - \hat{\tau}_{i0}) \\
    &= \sqrt{N}(\hat{\tau}_{i0} - \tau_i) + \sqrt{N} \frac{\sum_m \sum_{j \in I_m} \mathbf{1}(G_j = i)(\mathcal{Y}_j(\hat{\eta}_m)-\mathcal{Y}_j(\eta_0))}{\sum_j\mathbf{1}(G_j = i, S_j = 0)} \\
    &= \sqrt{N}(\hat{\tau}_{i0} - \tau_i) + \frac{\sum_m \frac{1}{\sqrt{N}}\sum_{j \in I_m} \mathbf{1}(G_j = i)(\mathcal{Y}_j(\hat{\eta}_m)-\mathcal{Y}_j(\eta_0))}{\frac{1}{N}\sum_j\mathbf{1}(G_j = i, S_j = 0)} \\
\end{align*}
The asymptotic distribution of the cross-fitted estimator therefore hinges on the asymptotic property of $\frac{1}{\sqrt{N}}\sum_{j \in I_m} \mathbf{1}(G_j = i)(\mathcal{Y}_j(\hat{\eta}_m)-\mathcal{Y}_j(\eta_0))$, which in turn depends on the convergence property of the nuisance function estimate $\hat{\eta}_m$ and its influence on the signal $\mathcal{Y}$. We therefore restate the last required assumption governing the convergence properties of $\hat{\eta}_m$:

\Nuisance*

Based on the assumptions above, we have the following bounds on the convergence rate of the signals when the nuisance function estimates are in the high-probability neighborhood, $\mathcal{T}_n$:

\begin{restatable}[Bounds on bias of signal]{lemma}{BiasBound}
    \label{lem:bias_bound}
    Under Assumptions \ref{asmp:rate_product} and \ref{asmp:bounded_nuisance}
    \[\sqrt{n}\sup_{\eta\in\mathcal{T}_n}\left|\E[\mathbf{1}(G = i)(\mathcal{Y}(\eta)-\mathcal{Y}(\eta_0))]\right| = o_P(1)\]
\end{restatable}

\begin{restatable}[Bounds on MSE of signal]{lemma}{MseBound}
    \label{lem:mse_bound}
    Under Assumptions \ref{asmp:bounded_p}, \ref{asmp:bounded_e}, \ref{asmp:bounded_variance}, \ref{asmp:rate} and \ref{asmp:bounded_nuisance}
    \[\sup_{\eta\in\mathcal{T}_n}\E\left|\mathbf{1}(G = i)(\mathcal{Y}(\eta)-\mathcal{Y}(\eta_0))\right|^2 = o_P(1)\]
\end{restatable}

which in turn implies that $\frac{1}{\sqrt{N}}\sum_{j \in I_m} \mathbf{1}(G_j = i)(\mathcal{Y}_j(\hat{\eta}_m)-\mathcal{Y}_j(\eta_0))$ converges to zero in probability:

\begin{restatable}[Numerator of difference is $o_P(1)$]{lemma}{NumeratorVanish} Under Assumptions \ref{asmp:bounded_p}, \ref{asmp:bounded_e}, \ref{asmp:bounded_variance} and \ref{asmp:nuisance},
    \label{lem:numerator_vanish}
    \[\frac{1}{\sqrt{N}}\sum_{j\in I_m} \mathbf{1}(G_j = i)(\mathcal{Y}_j(\hat{\eta}_m)-\mathcal{Y}_j(\eta_0)) = o_P(1), \;\;\;\; \forall m \in \{1,2,...,M\}\]
\end{restatable}

The proofs for Lemmas \ref{lem:bias_bound} to \ref{lem:numerator_vanish} are more labor-intensive and we defer these proofs to later subsections. Based on these lemmas, we arrive at the main result of Stage 2.

\begin{theorem}[Asymptotic normality of the cross-fitted transported GATE estimator] 
    \label{thm:normal}
    Under Assumptions \ref{asmp:bounded_p} -  \ref{asmp:nuisance},
    \[\sqrt{N}(\hat{\tau}_i - \tau_i) \xrightarrow[]{d} \mathcal{N}\left(0, \frac{\sigma^2_i- \tau_i^2\pi_g(i)(1-\pi_g(i))}{\pi_g(i)^2\P(G=i)}\right)\]
\end{theorem}

\begin{proof}
    \begin{align*}
        &\sqrt{N}(\hat{\tau}_i - \tau_i)&\\
        = &\sqrt{N}(\hat{\tau}_{i0} - \tau_i) + \sqrt{N}(\hat{\tau}_i - \hat{\tau}_{i0})& \\
        = &\sqrt{N}(\hat{\tau}_{i0} - \tau_i) + \sqrt{N} \frac{\sum_m \sum_{j \in I_m} \mathbf{1}(G_j = i)(\mathcal{Y}_j(\hat{\eta}_m)-\mathcal{Y}_j(\eta_0))}{\sum_j\mathbf{1}(G_j = i, S_j = 0)}& \\
        = &\sqrt{N}(\hat{\tau}_{i0} - \tau_i) + \frac{\sum_m \frac{1}{\sqrt{N}}\sum_{j \in I_m} \mathbf{1}(G_j = i)(\mathcal{Y}_j(\hat{\eta}_m)-\mathcal{Y}_j(\eta_0))}{\frac{1}{N}\sum_j\mathbf{1}(G_j = i, S_j = 0)}\genfrac{}{}{0pt}{}{\xrightarrow{p} 0\hfill}{\xrightarrow{p}\P(G=i,S=0)\hfill}&\genfrac{}{}{0pt}{}{\hfill \text{Lem. }\ref{lem:numerator_vanish}}{\hfill \text{WLLN}}\\
        \xrightarrow[]{d} &\mathcal{N}\left(0, \frac{\sigma^2_i- \tau_i^2\pi_g(i)(1-\pi_g(i))}{\pi_g(i)^2\P(G=i)}\right)& \text{Prop. }\ref{prop:oracle_normality}\text{, Slutsky's lemma}
    \end{align*}
\end{proof}

Note that \Cref{thm:normal} is simply \Cref{prop:normal}, which is the primary result of this section, with the variance explicitly stated. Thus, \Cref{prop:normal} is proven. 

\subsection{Proof for Lemmas \ref{lem:bias_bound} and \ref{lem:mse_bound}}
First, we prove Lemmas \ref{lem:bias_bound} and \ref{lem:mse_bound}, which will be necessary for Lemma \ref{lem:numerator_vanish}. Recall that Lemma \ref{lem:numerator_vanish} was essential for the proof of the asymptotic normality result in Theorem \ref{thm:normal}.

\BiasBound*
\MseBound*

\begin{proof}
We first define partial unweighted signal functionals for the two counterfactual outcomes
\begin{definition}[Partial unweighted signal functionals]
    \begin{align*}
        \mathcal{Y}^1(\eta) &\coloneqq \left[(1-S)g_1(X) + S\frac{1-p(X)}{p(X)}\frac{A(Y-g_1(X))}{e_1(X)}\right]\\
        \mathcal{Y}^0(\eta) &\coloneqq \left[(1-S)g_0(X) +  S\frac{1-p(X)}{p(X)}\frac{(1-A)(Y-g_0(X))}{e_0(X)}\right]\\
        \Rightarrow \mathcal{Y}(\eta) &= \mathcal{Y}^1(\eta) - \mathcal{Y}^0(\eta)
    \end{align*}
\end{definition}

At a high level, we will prove the above lemmas by decomposing the errors of signal functionals into simpler terms that can be bounded by standard concentration inequalities. This idea will be repeated for both the bias and MSE of the signals. To simplify the analysis, we can split up the unweighted signal into ``partial signals'' (for the treatment and control groups). Therefore, we set out to show the following lemmas: 
\begin{lemma}[Bounds on bias of partial signal]
     \label{lem:bias_bound_y10}
    Under Assumptions \ref{asmp:rate_product} and \ref{asmp:bounded_nuisance}
    \begin{align*}
        \sqrt{n}\sup_{\eta\in\mathcal{T}_n}\left|\E[\mathbf{1}(G = i)(\mathcal{Y}^1(\eta)-\mathcal{Y}^1(\eta_0))]\right| &= o_P(1)\\
        \sqrt{n}\sup_{\eta\in\mathcal{T}_n}\left|\E[\mathbf{1}(G = i)(\mathcal{Y}^0(\eta)-\mathcal{Y}^0(\eta_0))]\right| &= o_P(1)
    \end{align*}
\end{lemma}

\begin{lemma}[Bounds on MSE of partial signal]
    \label{lem:mse_bound_y10}
    Under Assumptions \ref{asmp:bounded_p}, \ref{asmp:bounded_e}, \ref{asmp:bounded_variance}, \ref{asmp:rate} and \ref{asmp:bounded_nuisance}
    \begin{align*}
        \sup_{\eta\in\mathcal{T}_n}\E\left|\mathbf{1}(G = i)(\mathcal{Y}^1(\eta)-\mathcal{Y}^1(\eta_0))\right|^2 &= o_P(1)\\
        \sup_{\eta\in\mathcal{T}_n}\E\left|\mathbf{1}(G = i)(\mathcal{Y}^0(\eta)-\mathcal{Y}^0(\eta_0))\right|^2 &= o_P(1)
    \end{align*}
\end{lemma}

In the following subsections, we prove the $\mathcal{Y}_1$ part of Lemmas \ref{lem:bias_bound_y10} and \ref{lem:mse_bound_y10}. The $\mathcal{Y}_0$ part will follow by symmetry. First, we further define $\eta(X) = \eta_0(X) + \delta_\eta(X)$, in detail:
\begin{align*}
    g_1(X) &= g_{10}(X) + \delta_{g_1}(X)\\
    p(X) &= p_0(X) + \delta_p(X)\\
    e(Z) &= e_0(X) + \delta_e(X)
\end{align*}
so that (omitting the parameter $X$ for brevity),
\begin{align*}
    &\mathcal{Y}^1(\eta)-\mathcal{Y}^1(\eta_0)\\
    =&\left[(1-S)(g_{10}+\delta_{g_1})+\frac{1-p_0-\delta_p}{p_0+\delta_p}\frac{SA(Y-g_{10}-\delta_{g_1})}{e_0 + \delta_e}\right]- \left[(1-S)g_{10}+\frac{1-p_0}{p_0}\frac{SA(Y-g_{10})}{e_0}\right]\\
    =&(1-S)\delta_{g_1}+\frac{1-p_0-\delta_p}{p_0+\delta_p}\frac{SA(Y-g_{10}-\delta_{g_1})}{e_0 + \delta_e}-\frac{1-p_0}{p_0}\frac{SA(Y-g_{10}-\delta_{g_1})}{e_0}-\frac{1-p_0}{p_0}\frac{SA}{e_0}\delta_{g_1}\\
    =&\left((1-S) -\frac{1-p_0}{p_0}\frac{SA}{e_0}\right)\delta_{g_1}+\left(\frac{1-p_0-\delta_p}{(p_0+\delta_p)(e_0+\delta_e)}-\frac{1-p_0}{p_0 e_0}\right)SA(Y-g_{10}-\delta_{g_1})\\
    =&\left((1-S) -\frac{1-p_0}{p_0}\frac{SA}{e_0}\right)\delta_{g_1}-\frac{e_0 \delta_p + (1-p_0)p_0\delta_e + (1-p_0)\delta_p \delta_e}{(p_0+\delta_p)(e_0+\delta_e)p_0 e_0}SA(Y-g_{10}-\delta_{g_1})\\
    =&\underbrace{\left((1-S) -\frac{1-p_0}{p_0}\frac{SA}{e_0}\right)\delta_{g_1}}_{S_1} - \underbrace{\frac{e_0 \delta_p + (1-p_0)p_0\delta_e + (1-p_0)\delta_p \delta_e}{(p_0+\delta_p)(e_0+\delta_e)p_0 e_0}SA(Y-g_{10})}_{S_2} \\
    & \qquad\qquad\qquad\qquad\qquad\qquad\qquad\qquad\qquad\qquad\qquad + \underbrace{\frac{e_0 \delta_p + (1-p_0)p_0\delta_e + (1-p_0)\delta_p \delta_e}{(p_0+\delta_p)(e_0+\delta_e)p_0 e_0}SA\delta_{g_1}}_{S_3}\\
    \coloneqq& S_1 - S_2 + S_3
\end{align*}

\subsubsection{Proof for Lemma \ref{lem:bias_bound_y10}}

For Lemma \ref{lem:bias_bound_y10} we want to bound 
\begin{align*}
    \left|\E \left[\mathbf{1}(G=i)(\mathcal{Y}^1(\eta)-\mathcal{Y}^1(\eta_0))\right]\right| &= \left|\E\left[\E\left[ \mathbf{1}(G=i)\left(\mathcal{Y}^1(\eta)-\mathcal{Y}^1(\eta_0)\right)\middle|X\right]\right]\right|&\\
    &= \left|\E\left[\mathbf{1}(G=i)\E\left[S_1 - S_2 + S_3\middle|X\right]\right]\right|& G\text{ is a function of }X\\
    &= \left|\E\left[\mathbf{1}(G=i)\left(\E\left[S_1|X\right] - \E\left[S_2|X\right] + \E\left[S_3\middle|X\right]\right)\right]\right|&
\end{align*}

For the term $\E[S_1|X]$,
\begin{align*}
  &\E[S_1|X]\\
  =&\E \left[\left((1-S) -\frac{1-p_0}{p_0}\frac{SA}{e_0}\right)\delta_{g_1}\middle|X\right]& \\
  =&\left((1-\E[S|X]) -\frac{1-p_0}{p_0}\frac{\E[SA|X]}{e_0}\right)\delta_{g_1}& \eta_0, \delta_{\eta} \text{ are functions of }X\\
  =&\left((1-p_0) -\frac{1-p_0}{p_0}\frac{p_0 e_0}{e_0}\right)\delta_{g_1} = 0 & \genfrac{}{}{0pt}{}{\hfill p_0(X) = \P[S=1|X]}{\hfill e_0(X) = \P[A=1|S=1,X]}\\
\end{align*}

For the term $\E[S_2|X]$
\begin{align*}
    &\E[S_2|X]&\\
    =&\E \left[ \frac{e_0 \delta_p + (1-p_0)p_0\delta_e + (1-p_0)\delta_p \delta_e}{(p_0+\delta_p)(e_0+\delta_e)p_0 e_0}SA(Y-g_{10}) \middle| X\right]&\\
    =&\frac{e_0 \delta_p + (1-p_0)p_0\delta_e + (1-p_0)\delta_p \delta_e}{(p_0+\delta_p)(e_0+\delta_e)p_0 e_0}\E[SA(Y-g_{10}) \mid X]& \eta_0, \delta_{\eta} \text{ are functions of }X\\
    =&\frac{e_0 \delta_p + (1-p_0)p_0\delta_e + (1-p_0)\delta_p \delta_e}{(p_0+\delta_p)(e_0+\delta_e)p_0 e_0} \cdot 0 = 0 & g_{10}(X) = \E[Y|S=1,A=1,X]
\end{align*}

For the term $\E[S_3|X]$,
\begin{align*}
    &\E[S_3|X]&\\
    =&\E \left[ \frac{e_0 \delta_p + (1-p_0)p_0\delta_e + (1-p_0)\delta_p \delta_e}{(p_0+\delta_p)(e_0+\delta_e)p_0 e_0}SA \delta_{g_1}\middle| X\right]&\\
    =&\frac{e_0 \delta_p + (1-p_0)p_0\delta_e + (1-p_0)\delta_p \delta_e}{(p_0+\delta_p)(e_0+\delta_e)p_0 e_0} \E\left[SA \middle| X\right] \delta_{g_1}& \eta_0, \delta_{\eta} \text{ are functions of }X\\
    =&\frac{e_0 \delta_p + (1-p_0)p_0\delta_e + (1-p_0)\delta_p \delta_e}{(p_0+\delta_p)(e_0+\delta_e)p_0 e_0} p_0 e_0 \delta_{g_1}& \genfrac{}{}{0pt}{}{\hfill p_0(X) = \P[S=1|X]}{\hfill e_0(X) = \P[A=1|S=1,X]}\\
    =&\frac{e_0 \delta_p + (1-p_0)p_0\delta_e + (1-p_0)\delta_p \delta_e}{p e}\delta_{g_1}&
\end{align*}

Therefore,
\begin{align*}
    &\left|\E \left[\mathbf{1}(G=i)[\mathcal{Y}^1(\eta)-\mathcal{Y}^1(\eta_0)]\right]\right|^2&\\
    = &\left|\E \left[\mathbf{1}(G=i)\left(\E[S_1|Z] - \E[S_2|Z] + \E[S_3|Z]\right) \right]\right|^2&\\
    = &\left|\E \left[\mathbf{1}(G=i)\E[S_3|Z]\right]\right|^2&\\
    \le & \left(\E\left|\mathbf{1}(G=i)\E[S_3|Z]\right|\right)^2 & |EA| \le E|A|\\
    \le & \left(\E\left|\E[S_3|Z]\right|\right)^2 &\\
    = &\left(\E\left|\frac{e_0 \delta_p + (1-p_0)p_0\delta_e + (1-p_0)\delta_p \delta_e}{p e}\delta_{g_1}\right|\right)^2 &\\
    \le &\left(\E\left|\frac{e_0 \delta_p}{p e}\delta_{g_1}\right|+\E\left|\frac{(1-p_0)p_0\delta_e}{p e}\delta_{g_1}\right|+ \E\left|\frac{(1-p_0)\delta_p \delta_e}{p e}\delta_{g_1}\right|\right)^2& \text{Triangular ineq.}\\ 
    \le & \bar{\mathcal{C}}^4 \left(\E\left|\delta_p\delta_{g_1}\right| + \E\left|\delta_e\delta_{g_1}\right| + \E\left|\delta_p \delta_e \delta_{g_1}\right|\right)^2 & \text{Assmp. }\ref{asmp:bounded_nuisance}\\
    = & \bar{\mathcal{C}}^4 \left(\E|\delta_p\delta_{g_1}| + \E|\delta_e\delta_{g_1}| + \frac{1}{2}\E|\delta_e||\delta_p \delta_{g_1}| + \frac{1}{2}\E|\delta_p||\delta_e \delta_{g_1}|\right)^2 &\\
    \le & \bar{\mathcal{C}}^4 \left(\E|\delta_p\delta_{g_1}| + \E|\delta_e\delta_{g_1}| + \frac{1}{2}\E|\delta_p \delta_{g_1}| + \frac{1}{2}\E|\delta_e \delta_{g_1}|\right)^2 & |\delta_p|, |\delta_e| \le 1\\
    = & \frac{9}{4} \bar{\mathcal{C}}^4 \left(\E|\delta_p\delta_{g_1}| + \E|\delta_e\delta_{g_1}|\right)^2 &\\
    \le & \frac{9}{4} \bar{\mathcal{C}}^4 \left(\sqrt{\E\delta_p^2}\sqrt{\E\delta_{g_1}^2} + \sqrt{\E\delta_e^2}\sqrt{\E\delta_{g_1}^2}\right)^2 & \text{H\"{o}lder's ineq.}
\end{align*}
So we have, 
\begin{align*}
    &\sqrt{n}\sup_{\eta\in\mathcal{T}_n}\left|\E \mathbf{1}(G=i)[Y^1(\eta)-Y^1(\eta_0)]\right|&\\
    \le &\sqrt{n}\sup_{\eta\in\mathcal{T}_n} \frac{3}{2} \bar{\mathcal{C}}^2 \left(\sqrt{\E\delta_p^2}\sqrt{\E\delta_{g_1}^2} + \sqrt{\E\delta_e^2}\sqrt{\E\delta_{g_1}^2}\right)&\\
    \le &\frac{3}{2} \bar{\mathcal{C}}^2 \sqrt{n} \mathbf{g}_N \left(\mathbf{p}_N + \mathbf{e}_N\right)&\text{Assump. }\ref{asmp:nuisance}&\\
    = &o_P(1) &\text{Assump. }\ref{asmp:rate_product}
\end{align*}

\subsubsection{Proof for Lemma \ref{lem:mse_bound_y10}}
\allowdisplaybreaks
Here we first place bounds on $\E S_1^2, \E S_2^2$ and $\E S_3^2$ for future use. For the term $\E S_1^2$, we have,
\begin{align*}
    &\E S_1^2&\\
    = &\E\left[\E\left[ \left(\left((1-S) -\frac{1-p_0}{p_0}\frac{SA}{e_0}\right)\delta_{g_1}\right)^2 \middle| X \right]\right]&\\
    = &\E \left[\left(1-p_0\right)^2 \E\left[\left(\frac{1-S}{1-p_0} -\frac{SA}{p_0 e_0}\right)^2\middle|X\right]\delta_{g_1}^2\right]&\eta_0, \delta_{\eta} \text{ are functions of }X\\
    = & \E \left[\left(1-p_0\right)^2  \E\left[\frac{(1-S)^2}{(1-p_0)^2} +\frac{(SA)^2}{(p_0 e_0)^2}\middle|X\right]\delta_{g_1}^2\right]&S\in\{0,1\}\\
    = & \E \left[\left(1-p_0\right)^2  \E\left[\frac{1-S}{(1-p_0)^2} +\frac{SA}{(p_0 e_0)^2}\middle|X\right]\delta_{g_1}^2\right]&1-S, SA \in \{0, 1\}\\    
    = & \E \left[\left(1-p_0\right)^2 \left(\frac{1}{1-p_0} +\frac{1}{p_0 e_0}\right)\delta_{g_1}^2\right]&\genfrac{}{}{0pt}{}{\hfill p_0(X) = \P[S=1|X]}{\hfill e_0(X) = \P[A=1|S=1,X]}\\
    \le & \frac{2}{\varepsilon_p \varepsilon_e} \E\delta_{g_1}^2 & \text{Assmp. }\ref{asmp:bounded_p}, \ref{asmp:bounded_e}
\end{align*}

We can similarly bound $\E S_2^2$, 
\begin{align*}
    &\E S_2^2&\\
    = &\E\left[\E\left[\left(\frac{e_0\delta_{p} + (1-p_0)p_0\delta_e + (1-p_0)\delta_{p}\delta_e}{(p_0+\delta_{p})(e_0 + \delta_e)p_0 e_0}SA(Y-g_{10})\right)^2 \middle| X \right]\right]&\\
    = &\E\left[\left(\frac{e_0\delta_{p} + (1-p_0)p_0\delta_e + (1-p_0)\delta_{p}\delta_e}{(p_0+\delta_{p})(e_0 + \delta_e)p_0 e_0}\right)^2 \E\left[ \left(SA(Y-g_{10})\right)^2 \middle| X \right]\right]& \eta_0, \delta_{\eta} \text{ are functions of }X\\
    = &\E\left[\left(\frac{e_0\delta_{p} + (1-p_0)p_0\delta_e + (1-p_0)\delta_{p}\delta_e}{(p_0+\delta_{p})(e_0 + \delta_e)p_0 e_0}\right)^2 \right.&S,A \in \{0,1\}\\
    & \qquad\qquad\qquad\qquad \left.\E\left[ (Y-g_{10})^2 \middle| X, S=1, A=1 \right]\P(S=1, A=1|X)\right]&\\
    \le &E\left[ \left(\frac{e_0\delta_{p} + (1-p_0)p_0\delta_e + (1-p_0)\delta_{p}\delta_e}{(p_0+\delta_{p})(e_0 + \delta_e)p_0 e_0}\right)^2\bar{\sigma}^2 \P(S=1, A=1|X)\right]& \text{Assmp. }\ref{asmp:bounded_variance}\\
    = &E\left[ \left(\frac{e_0\delta_{p} + (1-p_0)p_0\delta_e + (1-p_0)\delta_{p}\delta_e}{p e p_0 e_0}\right)^2\bar{\sigma}^2 p_0 e_0\right]& \genfrac{}{}{0pt}{}{\hfill p_0(X) = \P[S=1|X]}{\hfill e_0(X) = \P[A=1|S=1,X]}\\
    \le & \frac{\bar{\sigma}^2 \bar{C}^4}{\varepsilon_p \varepsilon_e} \E\left|e_0\delta_{p} + (1-p_0)p_0\delta_e + (1-p_0)\delta_{p}\delta_e\right|^2& \text{Assmp. }\ref{asmp:bounded_p}, \ref{asmp:bounded_e}, \ref{asmp:bounded_nuisance}\\
    \le & \frac{\bar{\sigma}^2 \bar{C}^4}{\varepsilon_p \varepsilon_e} \E\left[e_0|\delta_{p}| + (1-p_0)p_0|\delta_e| + (1-p_0)|\delta_{p}\delta_e|\right]^2& \text{Triangular ineq.}\\
    \le & \frac{\bar{\sigma}^2 \bar{C}^4}{\varepsilon_p \varepsilon_e} \E\left[|\delta_{p}| + |\delta_e| + |\delta_{p}\delta_e|\right]^2&\\
    = & \frac{\bar{\sigma}^2 \bar{C}^4}{\varepsilon_p \varepsilon_e} \E\left[|\delta_{p}| + |\delta_e| + \frac{1}{2}|\delta_{p}||\delta_e|+\frac{1}{2}|\delta_{p}||\delta_e|\right]^2&\\
    \le & \frac{\bar{\sigma}^2 \bar{C}^4}{\varepsilon_p \varepsilon_e} \E\left[|\delta_{p}| + |\delta_e| + \frac{1}{2}|\delta_{p}|+\frac{1}{2}|\delta_e|\right]^2&|\delta_p|,|\delta_e| \le 1\\
    = & \frac{9}{4}\frac{\bar{\sigma}^2 \bar{C}^4}{\varepsilon_p \varepsilon_e} \E\left[|\delta_{p}| + |\delta_e|\right]^2&\\
    = & \frac{9}{4}\frac{\bar{\sigma}^2 \bar{C}^4}{\varepsilon_p \varepsilon_e} \left[\E\delta_{p}^2 + \E\delta_e^2 + 2\E|\delta_{p}||\delta_e|\right]&\\
    \le & \frac{9}{4}\frac{\bar{\sigma}^2 \bar{C}^4}{\varepsilon_p \varepsilon_e} \left[\E\delta_{p}^2 + \E\delta_e^2 + 2\sqrt{\E\delta_{p}^2}\sqrt{\E\delta_e^2}\right]&\text{Cauchy-Schwartz}\\
    =&\frac{9}{4}\frac{\bar{\sigma}^2 \bar{C}^4}{\varepsilon_p \varepsilon_e} \left(\sqrt{\E\delta_{p}^2} + \sqrt{\E\delta_e^2}\right)^2&
\end{align*}

Finally, we bound $\E S_3^2$, 
\begin{align*}
    &\E S_3^2 &\\
    = &\E\left[\E\left[\left(\frac{e_0\delta_{p} + (1-p_0)p_0\delta_e + (1-p_0)\delta_{p}\delta_e}{(p_0+\delta_{p})(e_0 + \delta_e)p_0 e_0}SA\delta_{g_1}\right)^2 \middle| X \right]\right]&\\
    = &\E\left[ \left(\frac{e_0\delta_{p} + (1-p_0)p_0\delta_e + (1-p_0)\delta_{p}\delta_e}{(p_0+\delta_{p})(e_0 + \delta_e)p_0 e_0}\delta_{g_1}\right)^2 \E\left[ (SA)^2 \middle| X \right]\right]& \eta_0, \delta_{\eta} \text{ are functions of }X\\
    = &\E\left[\left(\frac{e_0\delta_{p} + (1-p_0)p_0\delta_e + (1-p_0)\delta_{p}\delta_e}{(p_0+\delta_{p})(e_0 + \delta_e)p_0 e_0}\delta_{g_1}\right)^2 \E\left[ SA \middle| X \right]\right]& SA \in \{0, 1\}\\
    = &\E\left[\left(\frac{e_0\delta_{p} + (1-p_0)p_0\delta_e + (1-p_0)\delta_{p}\delta_e}{p e p_0 e_0}\delta_{g_1}\right)^2 p_0 e_0\right]& \genfrac{}{}{0pt}{}{\hfill p_0(X) = \P[S=1|X]}{\hfill e_0(X) = \P[A=1|S=1,X]}\\
    \le & \frac{\bar{C}^4}{\varepsilon_p \varepsilon_e} \E\left|\left(e_0\delta_{p} + (1-p_0)p_0\delta_e + (1-p_0)\delta_{p}\delta_e\right)\delta_{g_1}\right|^2& \text{Assmp. }\ref{asmp:bounded_p}, \ref{asmp:bounded_e}, \ref{asmp:bounded_nuisance}\\
    = & \frac{\bar{C}^4}{\varepsilon_p \varepsilon_e} \E\left[|e_0\delta_{p} + (1-p_0) p_0 \delta_e + (1-p_0)\delta_{p}\delta_e|^2|\delta_{g_1}|^2\right]&\\
    \le & \frac{\bar{C}^4}{\varepsilon_p \varepsilon_e} \E\left[\left(|e_0\delta_{p}| + |(1-p_0) p_0 \delta_e| + |(1-p_0)\delta_{p}\delta_e|\right)^2|\delta_{g_1}|^2\right]& \text{Triangular ineq.}\\
    \le & \frac{9\bar{C}^4}{\varepsilon_p \varepsilon_e} \E\delta_{{g_1}}^2& 0 \le p_0, e_0, |\delta_p|, |\delta_e| \le 1
\end{align*}

From the above, we have
\begin{align*}
    &\E\left|\mathbf{1}(G = i)(\mathcal{Y}^1(\eta)-\mathcal{Y}^1(\eta_0))\right|^2&\\
    =&\E|\mathbf{1}(G = i)||S_1-S_2+S_3|^2&\\
    \le&\E|S_1-S_2+S_3|^2&\\
    \le & \E(|S_1|+|S_2|+|S_3|)^2&\text{Triangular ineq.}\\
    = &\left[\E S_1^2 + \E S_2^2 + \E S_3^2 + 2\E|S_1 S_2| + 2\E|S_1 S_3| + 2\E|S_2 S_3|\right]&\\
    \le &\left[\E S_1^2 + \E S_2^2 + \E S_3^2 +\right.&\\
    &\qquad\; \left.2\sqrt{\E |S_1|^2}\sqrt{\E |S_2|^2} + 2\sqrt{\E |S_1|^2}\sqrt{\E |S_3|^2} + 2\sqrt{\E |S_2|^2}\sqrt{\E |S_3|^2}\right]&\text{Cauchy-Schwartz}\\
    = &\left[\sqrt{\E S_1^2}+\sqrt{\E S_1^2}+\sqrt{\E S_1^2}\right]^2&\\
    \le &\left[\sqrt{\frac{2}{\varepsilon_p \varepsilon_e}}\sqrt{\E \delta_{g_1}^2} + \sqrt{\frac{9}{4}\frac{\bar{\sigma}^2 \bar{C}^4}{\varepsilon_p \varepsilon_e}} \left(\sqrt{\E\delta_{p}^2} + \sqrt{\E\delta_e^2}\right) + \sqrt{\frac{9\bar{C}^4}{\varepsilon_p \varepsilon_e}} \sqrt{\E\delta_{{g_1}}^2}\right]^2&\\
    \le &C\left[\sqrt{\E \delta_{{g_1}}^2} + \sqrt{\E\delta_{p}^2} + \sqrt{\E\delta_e^2}\right]^2&C\coloneqq \frac{\left(3\bar{C}^2+\sqrt{2}\right)^2}{\varepsilon_p \varepsilon_e} \vee \frac{9}{4}\frac{\bar{\sigma}^2\bar{C}^4}{\varepsilon_p \varepsilon_e}
\end{align*}
So we have,
\begin{align*}
    &\sup_{\eta\in\mathcal{T}_n}\E\left|\mathbf{1}(G = i)(\mathcal{Y}^1(\eta)-\mathcal{Y}^1(\eta_0))\right|^2&\\
    \le &C\sup_{\eta\in\mathcal{T}_n}\left[\sqrt{\E \delta_{{g_1}}^2} + \sqrt{\E\delta_{p}^2} + \sqrt{\E\delta_e^2}\right]^2&\\
    \le &C\left[\sup_{\eta\in\mathcal{T}_n}\sqrt{\E \delta_{{g_1}}^2} + \sup_{\eta\in\mathcal{T}_n}\sqrt{\E\delta_{p}^2} + \sup_{\eta\in\mathcal{T}_n}\sqrt{\E\delta_e^2}\right]^2&\\
    \le &C\left[\mathbf{g}_n + \mathbf{p}_n + \mathbf{e}_n\right]^2&\text{Assmp. }\ref{asmp:nuisance}\\
    \le &9C\left[\mathbf{g}_n \vee \mathbf{p}_n \vee \mathbf{e}_n\right]^2 = o_P(1)&\text{Assmp. }\ref{asmp:rate}
\end{align*}

\subsubsection{Assembling the proofs for Lemmas \ref{lem:bias_bound} and \ref{lem:mse_bound}}
For Lemma \ref{lem:bias_bound}:
\begin{align*}
    &\sqrt{n}\sup_{\eta\in\mathcal{T}_n}\left|\E[\mathbf{1}(G = i)(\mathcal{Y}(\eta)-\mathcal{Y}(\eta_0))]\right|&\\
    =&\sqrt{n}\sup_{\eta\in\mathcal{T}_n}\left|\E[\mathbf{1}(G = i)(\mathcal{Y}^1(\eta)-\mathcal{Y}^1(\eta_0))] - \E[\mathbf{1}(G = i)(\mathcal{Y}^0(\eta)-\mathcal{Y}^0(\eta_0))]\right|&\\
    \le&\sqrt{n}\sup_{\eta\in\mathcal{T}_n}\left\{\left|\E[\mathbf{1}(G = i)(\mathcal{Y}^1(\eta)-\mathcal{Y}^1(\eta_0))]\right| + \left|\E[\mathbf{1}(G = i)(\mathcal{Y}^0(\eta)-\mathcal{Y}^0(\eta_0))]\right|\right\}&\text{Triangular ineq.}\\
    \le&\sqrt{n}\sup_{\eta\in\mathcal{T}_n}\left|\E[\mathbf{1}(G = i)(\mathcal{Y}^1(\eta)-\mathcal{Y}^1(\eta_0))]\right| + \sqrt{n}\sup_{\eta\in\mathcal{T}_n}\left|\E[\mathbf{1}(G = i)(\mathcal{Y}^0(\eta)-\mathcal{Y}^0(\eta_0))]\right| = o_P(1)&\text{Lem. }\ref{lem:bias_bound_y10}
\end{align*}
For Lemma \ref{lem:mse_bound}:
\begin{align*}
    &\sup_{\eta\in\mathcal{T}_n}\E\left|\mathbf{1}(G = i)(\mathcal{Y}(\eta)-\mathcal{Y}(\eta_0))\right|^2&\\
    =&\sup_{\eta\in\mathcal{T}_n}\E\left|\mathbf{1}(G = i)(\mathcal{Y}^1(\eta)-\mathcal{Y}^1(\eta_0)) - \mathbf{1}(G = i)(\mathcal{Y}^0(\eta)-\mathcal{Y}^0(\eta_0))\right|^2&\\
    \le&\sup_{\eta\in\mathcal{T}_n}\E\left\{\left|\mathbf{1}(G = i)(\mathcal{Y}^1(\eta)-\mathcal{Y}^1(\eta_0))\right| + \left|\mathbf{1}(G = i)(\mathcal{Y}^0(\eta)-\mathcal{Y}^0(\eta_0))\right|\right\}^2&\text{Triangular ineq.}\\
    =&\sup_{\eta\in\mathcal{T}_n}\left\{\E\left|\mathbf{1}(G = i)(\mathcal{Y}^1(\eta)-\mathcal{Y}^1(\eta_0))\right|^2 + \E\left|\mathbf{1}(G = i)(\mathcal{Y}^0(\eta)-\mathcal{Y}^0(\eta_0))\right|^2 + \right.&\\
    &\qquad\qquad \left.2\E\left|\mathbf{1}(G = i)(\mathcal{Y}^1(\eta)-\mathcal{Y}^1(\eta_0))\right|\left|\mathbf{1}(G = i)(\mathcal{Y}^0(\eta)-\mathcal{Y}^0(\eta_0))\right|\right\}&\\
    \le&\sup_{\eta\in\mathcal{T}_n}\left\{\E\left|\mathbf{1}(G = i)(\mathcal{Y}^1(\eta)-\mathcal{Y}^1(\eta_0))\right|^2 + \E\left|\mathbf{1}(G = i)(\mathcal{Y}^0(\eta)-\mathcal{Y}^0(\eta_0))\right|^2 + \right.&\text{Cauchy-Schwartz}\\
    &\qquad\qquad \left.2\sqrt{\E\left|\mathbf{1}(G = i)(\mathcal{Y}^1(\eta)-\mathcal{Y}^1(\eta_0))\right|^2}\sqrt{\E\left|\mathbf{1}(G = i)(\mathcal{Y}^0(\eta)-\mathcal{Y}^0(\eta_0))\right|^2}\right\}&\\
    \le&\sup_{\eta\in\mathcal{T}_n}\E\left|\mathbf{1}(G = i)(\mathcal{Y}^1(\eta)-\mathcal{Y}^1(\eta_0))\right|^2 + \sup_{\eta\in\mathcal{T}_n}\E\left|\mathbf{1}(G = i)(\mathcal{Y}^0(\eta)-\mathcal{Y}^0(\eta_0))\right|^2 + &\\
    &\qquad\qquad 2\sqrt{\sup_{\eta\in\mathcal{T}_n}\E\left|\mathbf{1}(G = i)(\mathcal{Y}^1(\eta)-\mathcal{Y}^1(\eta_0))\right|^2}\sqrt{\sup_{\eta\in\mathcal{T}_n}\E\left|\mathbf{1}(G = i)(\mathcal{Y}^0(\eta)-\mathcal{Y}^0(\eta_0))\right|^2}&\\
    =&o_P(1)&\text{Lem. } \ref{lem:mse_bound_y10}
\end{align*}
\end{proof}

\subsection{Proof for Lemma \ref{lem:numerator_vanish}}
Now that we have shown Lemmas \ref{lem:bias_bound} and \ref{lem:mse_bound}, it remains to show Lemma \ref{lem:numerator_vanish}. 

\NumeratorVanish*
\allowdisplaybreaks
\begin{proof}
    First we observe
    \begin{align*}
        &\frac{1}{\sqrt{N}}\sum_{j\in I_m} \mathbf{1}(G_j = i)(\mathcal{Y}_j(\hat{\eta}_m)-\mathcal{Y}_j(\eta_0))&\\
        = &\frac{1}{\sqrt{N}}\left[\sum_{j\in I_m} \mathbf{1}(G_j = i)(\mathcal{Y}_j(\hat{\eta}_m)-\mathcal{Y}_j(\eta_0)) - \E[\mathbf{1}(G = i)(\mathcal{Y}(\hat{\eta}_m)-\mathcal{Y}(\eta_0))]\right] +&\\
        &\qquad\qquad\qquad\qquad\qquad\qquad\qquad\qquad\qquad \frac{1}{\sqrt{N}}\sum_{j\in I_m}\E[\mathbf{1}(G = i)(\mathcal{Y}(\hat{\eta}_m)-\mathcal{Y}(\eta_0))] & \\
        = &\underbrace{\frac{N_M}{\sqrt{N}}\left[\left(\frac{1}{N_M}\sum_{j\in I_m} \mathbf{1}(G_j = i)(\mathcal{Y}_j(\hat{\eta}_m)-\mathcal{Y}_j(\eta_0))\right) - \E[\mathbf{1}(G = i)(\mathcal{Y}(\hat{\eta}_m)-\mathcal{Y}(\eta_0))]\right]}_{R_1(m)} +&\\
        &\qquad\qquad\qquad\qquad\qquad\qquad\qquad\qquad\qquad\qquad\qquad \underbrace{\frac{N_M}{\sqrt{N}}\E[\mathbf{1}(G = i)(\mathcal{Y}(\hat{\eta}_m)-\mathcal{Y}(\eta_0))]}_{R_2(m)} & \\
        \coloneqq & R_1(m) + R_2(m)
    \end{align*}
    We define the event $\mathcal{E}_N$ as $\cap_k \left(\hat{\eta}_m \in \mathcal{T}_{N_M^c}\right)$, i.e. all $M$ nuisance function estimates falling into the high-probability neighborhood where Lemmas \ref{lem:bias_bound} and \ref{lem:mse_bound} apply. From union bound,
    \begin{equation*}
        1-\P(\mathcal{E}_N) \le \sum_k \P(\hat{\eta}_m \notin \mathcal{T}_{N_M^c}) \ \le K\epsilon_{N_M^c} = o_P(1) \qquad\qquad\qquad\qquad \because \epsilon_n = o_P(1)
    \end{equation*}
    
    Conditional on $\mathcal{E}_N$ and the data complementary to fold $m$, which we denote as $D_m$, we have for any $\epsilon > 0$,
    \begin{align*}
        &\P(|R_1(m)| \ge \epsilon|\mathcal{E}_N,D_m)&\\
        = &\P\left(\left|\left(\frac{1}{N_M}\sum_{j\in I_m} \mathbf{1}(G_j = i)(\mathcal{Y}_j(\hat{\eta}_m)-\mathcal{Y}_j(\eta_0))\right) - \right.\right.&\\
        & \qquad\qquad\qquad\qquad \left.\left.\phantom{\frac{1}{1}}\E[\mathbf{1}(G = i)(\mathcal{Y}(\hat{\eta}_m)-\mathcal{Y}(\eta_0))]\right| \ge \frac{\sqrt{N}}{N_M}\epsilon\middle| \mathcal{E}_N, D_m\right)&\\
        \le & \frac{N_M^2}{N\epsilon^2} var\left[\frac{1}{N_M}\sum_{j\in I_m} \mathbf{1}(G_j = i)(\mathcal{Y}_j(\hat{\eta}_m)-\mathcal{Y}_j(\eta_0))\middle|\mathcal{E}_N,D_m\right]&\text{Chebyshev ineq.}\\
        = & \frac{N_M}{N\epsilon^2} var\left[\mathbf{1}(G = i)(\mathcal{Y}(\hat{\eta}_m)-\mathcal{Y}(\eta_0))\middle|\mathcal{E}_N,D_m\right]&\\
        \le & \frac{1}{M\epsilon^2}\E\left[(\mathbf{1}(G = i)(\mathcal{Y}(\hat{\eta}_m)-\mathcal{Y}(\eta_0)))^2\middle|\mathcal{E}_N,D_m\right] &\\
        \le & \frac{1}{M\epsilon^2}\sup_{\eta\in\mathcal{T}_{N_M}}\E\left[\left(\mathbf{1}(G = i)(\mathcal{Y}(\eta)-\mathcal{Y}(\eta_0))\right)^2\middle|D_m\right] &\hat{\eta}_m \in \mathcal{T}_{N_M^c}\text{ under }\mathcal{E}_N\\
        \le & \frac{1}{M\epsilon^2}\sup_{\eta\in\mathcal{T}_{N_M}}\E\left(\mathbf{1}(G = i)(\mathcal{Y}(\eta)-\mathcal{Y}(\eta_0))\right)^2 &\text{Fold }m\text{ is independent of }D_m\\
        = &\frac{1}{M\epsilon^2}o_P(1) = o_P(1) &\text{Lem. }\ref{lem:mse_bound}
    \end{align*}
    
    Also, conditional on $\mathcal{E}_N$ and $D_m$
    \begin{align*}
        \left|R_2(m)\right| = &\left|\frac{N_M}{\sqrt{N}\sqrt{N_M^c}}\sqrt{N_M^c}\E[\mathbf{1}(G = i)(\mathcal{Y}(\hat{\eta}_m)-\mathcal{Y}(\eta_0))]\middle|\mathcal{E}_N, D_m\right|&\\
        = &\frac{1}{\sqrt{M-1}}\sqrt{N_M^c}\left|\E[\mathbf{1}(G = i)(\mathcal{Y}(\hat{\eta}_m)-\mathcal{Y}(\eta_0))]|\mathcal{E}_N, D_m\right|&\\
        \le & \frac{1}{\sqrt{M-1}}\sqrt{N_M^c}\sup_{\eta \in \mathcal{T}_{N_M^c}}\left|\E[\mathbf{1}(G = i)(\mathcal{Y}(\eta)-\mathcal{Y}(\eta_0))|D_m]\right|&\hat{\eta}_m \in \mathcal{T}_{N_M^c}\text{ under }\mathcal{E}_N\\
        = & \frac{1}{\sqrt{M-1}}\sqrt{N_M^c}\sup_{\eta \in \mathcal{T}_{N_M^c}}\left|\E[\mathbf{1}(G = i)(\mathcal{Y}(\eta)-\mathcal{Y}(\eta_0))]\right|&\text{Fold }m\text{ is independent of }D_m\\
        = &\frac{1}{\sqrt{M-1}}o_P(1) = o_P(1)& \text{Lem. } \ref{lem:bias_bound}
    \end{align*}
    
    Which implies that, conditional on $\mathcal{E}_N$ and $D_m$, $R_1(m)+R_2(m)=o_P(1)$. So for any $\epsilon > 0$:
    \begin{align*}
        &\P\left(\left|\frac{1}{\sqrt{N}}\sum_{j\in I_m} \mathbf{1}(G_j = i)(\mathcal{Y}_j(\hat{\eta}_m)-\mathcal{Y}_j(\eta_0))\right|>\epsilon\right)\\
        =&\P(|R_1(m)+R_2(m)|\ge\epsilon)\\
        =&\P(|R_1(m)+R_2(m)|\ge\epsilon|\mathcal{E}_N)\P(\mathcal{E}_N) + \P(|R_1(m)+R_2(m)|\ge\epsilon|\bar{\mathcal{E}}_N)(1-\P(\mathcal{E}_N))\\
        \le &\P(|R_1(m)+R_2(m)|\ge\epsilon|\mathcal{E}_N) + (1-\P(\mathcal{E}_N))\\
        = &\int\P(|R_1(m)+R_2(m)|\ge\epsilon|\mathcal{E}_N,D_m)d\P(D_m|\mathcal{E}_N) +(1-\P(\mathcal{E}_N))&\\
        = &o_P(1)+o_P(1) = o_P(1)
    \end{align*}
    and the lemma is proven.
\end{proof}

\section{Details on WHI Experiments}
\label{sec:whi_exp_details} 

We assess our algorithm on clinical trial data and observational data available from the Women’s Health Initiative (WHI). The RCTs were run by the WHI via 40 US clinical centers from 1993-2005 (1993-1998: enrollment + randomization; 2005: end of follow-up) on postmenopausal women aged 50-79 years, and the observational dataset was designed and run in parallel on a similar population. Note that this data is publicly available to researchers and requires only an application on BIOLINCC (https://biolincc.nhlbi.nih.gov/studies/whi\textunderscore ctos/).

\subsection{Data} 
\textbf{WHI RCT} – There are three clinical trials associated with the WHI. The RCT that we will be leveraging in this set of experiments is the Postmenopausal Hormone Therapy (PHT) trial, which was run on postmenopausal women aged 50-79 years who had an intact uterus. This trial included a total of $N_{HT} = 16608$ patients. The intervention of interest was a hormone combination therapy of estrogen and progesterone. Specifically, post-randomization, the treatment group was given 2.5 mg of medroxyprogesterone as well as 0.625 mg of estrogen a day. The control group was given a placebo. Finally, there are several outcomes that were tracked and studied in the principal analysis done on this trial \citep{rossouw2002risks}. These outcomes are of three broad categories: a) cardiovascular events, including coronary heart disease, which served as a primary endpoint b) cancer (e.g. endometrial, breast, colorectal, etc.), and c) fractures.

\textbf{WHI OS} – The observational study component of the WHI tracked the medical events and health habits of $N = 93676$ women. Recruitment for the study began in 1994 and participants were followed until 2005, i.e. a similar follow-up to the RCT. Follow-up was done in a similar fashion as in the RCT (i.e. patients would have annual visits, in addition to a “screening” visit, where they would be given survey forms to fill out to track any events/outcomes). Thus, the same outcomes, including cancers, fractures, and cardiovascular events, are tracked in the observational study.

\subsection{Outcome}
The outcome of interest in our analysis is a “global index”, which is a summary statistic of several outcomes, including coronary heart disease, stroke, pulmonary embolism, endometrial cancer, colorectal cancer, hip fracture, and death due to other causes. Events or outcomes are tracked for each patient, and are recorded as “day of event/outcome” in the data, where the initial time-point for follow-up is the same for both the RCT and OS. At a high level, the “global index” is essentially the minimum “event day” when considering all the previously mentioned events.

We binarize the “global index,” by choosing a time point, $t$, before the end of follow-up and letting $Y=1$ if the observed event day is before $t$ and $Y=0$ otherwise. Thus, we are looking at whether the patient will experience the event within some particular period of time or not. We set $t=7$ years. Note that we sidestep censorship of a patient before the threshold by defining the outcomes in the following way: $Y=1$ indicates that a patient is observed to have the event before the threshold, and $Y=0$ indicates that a patient is not observed to have the event before the threshold. We apply this binarization in the same way for both the RCT and OS. Extending our method to a survival analysis framing is beyond the scope of this paper, but an interesting direction for future work.

\subsection{Intervention}
Recall from above that the intervention studied in the RCT was 2.5mg of medroxyprogesterone + 0.625 mg of estrogen and the control was a placebo pill. The RCT was run as an “intention-to-treat” trial. To establish “treatment” and “control” groups in the OS, we leverage the annual survey data collected from patients and assign a patient to the treatment group if they confirm usage of both estrogen and progesterone in the first three years. A patient is assigned to the control group if they deny usage of both estrogen and progesterone in the first three years. We exclude a patient from the analysis if she confirms usage of one and not the other OR if the field in the survey is missing OR if they take some other hormone therapy. We end up with a total of $N_{obs} = 33511$ patients.

\subsection{Data Processing + Covariates}
We use only covariates that are measured both in the RCT and OS to simplify the analysis. Because this information is gathered via the same set of questionnaires, they each indicate the same type of covariate. In other words, there is consistency of meaning across the same covariates across the RCT and the OS. We end up with a total of 1576 covariates.

\subsection{Details of Experimental Setup}
We give a more detailed exposition of the steps in our experimental workflow, which were described in brief in the main paper. 

\begin{itemize}
    \item \textbf{Step 0}: \textit{Replicate the principal results from the PHT trial, given in Table 2 of \citep{rossouw2002risks}, using the WHI OS data.} In this step, we fit a doubly robust estimator of the style given in Appendix \ref{sec:experimentdetails}.
    
    \item \textbf{Step 1}: \textit{While treating the WHI OS dataset as the “unbiased” observational dataset (hence the need for Step 0), simulate additional “biased” observational datasets by inducing bias into the WHI OS.} We construct four additional “biased” datasets (for a total of five observational datasets, including the WHI OS dataset), where we use the following procedure to induce selection bias – of the people who were not exposed to the treatment and did not end up getting the event, we drop each person with some probability, $p$. We set $p = [0.1,0.3,0.5,0.7]$ to get the four additional observational datasets.

    This type of selection bias may reflect the following clinical scenario: consider a patient who is relatively healthy who does not end up taking any hormone therapy. This patient might enroll initially in the OS, but may drop out or stop responding to the surveys. If the committee running the study does not explicitly account for this drop-out rate, then the resultant study will suffer from selection bias. \citep{banack2019investigating} detail additional examples of selection bias that can occur in observational studies.
    Importantly, this part is the only part of our setup that involves any simulation. However, in order to properly evaluate our method, we need to know which datasets are biased and unbiased in our set. Thus, we opt to simulate the bias.

    \item \textbf{Step 2}:  \textit{Run our procedure over “multiple tasks,” generating confidence intervals on the treatment effect for different subgroups.} To do so, we compile a list of covariates, taking both from \citep{schnatz2017effects} as well as covariates with high feature importance in both the propensity score model and response surface model from the estimator in Step 0. We generate all pairs from this list and use each pair to generate four subgroups. We treat two of the subgroups as validation subgroups and two of them as extrapolated subgroups in that we “hide” the RCT data in those subgroups when fitting our doubly robust transported estimator. (This gives us the benefit of knowing the RCT result for the extrapolated subgroups, which is useful in evaluation). Pairs that do not have enough support (threshold of 400 observations) in each group are removed. The total number of “tasks” (or covariate pairs) that we have is 592 (and therefore 2368 subgroups).    
    
    \item \textbf{Step 3}: \textit{Evaluate ExPCS (our method), ExOCS, Simple, and Meta-Analysis for each of the covariate pairs.} Additionally, we evaluate an “oracle” method, which always selects only the original observational study (i.e. the base WHI OS to which we have not added any selection bias) and reports the interval estimate computed on this study. To evaluate these methods, we will treat the RCT point estimates as “correct.” For each, we compute the following metrics: \textbf{Length} – length of the confidence interval for the subgroup; \textbf{Coverage} – percentage of tasks for which the method’s interval covers the RCT point estimate; \textbf{Unbiased OS Percentage} – across all tasks, the percentage at which our approach retains the unbiased study after the falsification step.

\end{itemize}

Note that we utilize sample splitting when running the above procedure. Namely, we use $50\%$ of the data as a “training” set, where we experiment with different classes of covariates and different types of bias, and then reserve $50\%$ of the data as a “testing” set, on which we do the final run of the analysis and report results. All nuisance functions in the doubly robust estimator are fit with a Gradient Boosting Classifier with significant regularization. In practice, we found that any highly-regularized tree-based model works well. 

\subsection{Covariate List for Task Generation}
\label{sec:cov-task}

Below is the list of covariates used to generate the tasks in \textbf{Step 2} of our experiment: 
\begin{itemize}
    \item ALCNOW (current alcohol user)
    \item BMI $\le 30$
    \item BLACK 
    \item SMOKING (current smoker)
    \item DIAB (diabetes ever)
    \item HYPT (hypertension ever)
    \item BRSTFEED (breastfeeding) 
    \item MSMINWK $\le 106$ (minutes of moderate to strenuous activity per week)
    \item BRSTBIOP (breast biopsy done)
    \item RETIRED
    \item EMPLOYED
    \item OC (oral contraceptive use ever)
    \item LIVPRT (live with husband or partner)
    \item MOMALIVE (natural mother still alive)
    \item LATREGION-Northern $> 40$ degrees north
    \item BKBONE (broke bone ever)
    \item NUMFALLS
    \item GRAVID (gravidity)
    \item AGE $\le 64$
    \item ANYMENSA $\le 51$ (age at last bleeding)
    \item MENOPSEA $\le 50$ (age at last regular period)
    \item MENO $\le$51 (age at menopause)
    \item LSTPAPDY (days from randomization to last pap smear)
    \item BMI $\le 27.7$
    \item TMINWK $\le 191$ (minutes of recreational exercise per week) 
    \item HEIGHT $\le 161$
    \item WEIGHT $\le 72$
    \item WAIST $\le 86$
    \item HIP $\le 105$
    \item WHR $\le 0.81$ (waist to hip ratio)
    \item TOTHCAT (HRT duration by category)
    \item MEDICARE (on medicare)
    \item HEMOGLBN $\le 13$
    \item PLATELET $\le 244$
    \item WBC $\le 6$
    \item HEMATOCR $\le 40$
\end{itemize}

\section{Details on Semi-Synthetic Experiment (Data Generation and Model Hyperparameters)}
\label{sec:exp_details}
\subsection{Data Generation}
For each simulated dataset, we generate 1 RCT and $K$ observational studies. The RCT is assumed to have covariate values identical to the IHDP dataset but is restricted to infants with married mothers. For the observational studies, we resample the rows of the IHDP dataset to the desired sample size $n = rn_0$. The covariate distribution of the observational studies are made different from the RCT by weighted sampling, with the relative weights set as
\[w = 0.8^{\mathbf{1}(\text{male infant})+\mathbf{1}(\text{mother smoked})+\mathbf{1}(\text{mother worked during pregnancy})}\]
Then, to introduce confounding (in the observational data), we generate $m_c$ continuous confounders and $m_b$ binary confounders. Each continuous confounder is drawn from a mixture of $0.5\cN(0,1) + 0.5\cN(3,1)$ in the RCT and $(0.25+0.5A)\cN(3,1) + (0.75-0.5A)\cN(0,1)$ in the observational studies, where $A$ is the treatment indicator. Similarly, each binary confounder is drawn from $\text{Bern}(0.5)$ in the RCT and $\text{Bern}(0.25+0.5A)$ in the observational studies. In the following, we denote the covariate vector as $X\in \mathbb{R}^{m_x}$ where $m_x = 28$ is the number of covariates in the IHDP dataset, and the generated confounder vector as $Z\in \mathbb{R}^{(m_c+m_b)}$. For brevity, we also denote the vector $(A, X^\top)^\top$ as $\tilde{X}$.

For outcome simulation in the datasets, we modify \textit{response surface B} from \citet{hill2011bayesian} to account for additional confounding variables. We set the following counterfactual outcome distributions:
\begin{align*}
Y_0 &\sim \cN\left(\exp\left[\left(\tilde{X}+\frac{1}{2}\mathbf{1}\right)^\top\beta\right] + Z^\top\gamma, 1\right)\\
Y_1 &\sim \cN(\tilde{X}^\top\beta + Z^\top\gamma + \omega, 1),
\end{align*}
where $\mathbf{1} \in \mathbb{R}^{(m_x+1)}$ is vector of ones, $\beta\in \mathbb{R}^{(m_x+1)}$ is a vector where each element is randomly sampled from $(0, 0.1, 0.2, 0.3, 0.4)$ with probabilities $(0.6, 0.1, 0.1, 0.1, 0.1)$, $\gamma\in \mathbb{R}^{(m_c+m_b)}$ is a vector where each element is randomly sampled from $(0.1, 0.2, 0.5, 0.75, 1)$ with uniform probability, and $\omega = 23$ is a constant chosen to limit the size of the GATEs. The observed outcome is then $Y \coloneqq AY_1 + (1-A)Y_0$. We then conceal a number of confounders, chosen in order from the highest to lowest weighted, from each observational study to mimic the scenario of unobserved confounding. 
The number of concealed confounders in each observational study is denoted as $\mathbf{c_z} = (c_{z1}, c_{z2},..., c_{zK})$.

\subsection{Hyperparameters}
\subsubsection{Logistic regression}
\begin{tabular}{c|c}
    Hyperparametes & Value set \\
    \hline
    Penalty type & $\ell_2$ \\
    Penalty coefficient & $\{1, 0.1, 0.01, 0.001\}$
\end{tabular}

\subsubsection{Multilayer perceptron regression}
\begin{tabular}{c|c}
    Hyperparametes & Value set \\
    \hline
    \makecell{$\#$ of hidden layers \\ and $\#$ of perceptrons} & $[1,(100)], [2,(50,50)], [2,(25,25)]$ \\
    Activation function & ReLU, tanh\\
    Solver & Adam \\
    Alpha & $(1, 0.1, 0.01, 0.001, 0.0001)$\\
    Learning rate & 0.001\\
    $\#$ of epochs & $(250,500)$
\end{tabular}

\section{Additional Semi-Synthetic Experimental Results}
\label{sec:add_experiments}
\begin{figure}[h!]
     \begin{subfigure}[t]{0.55\textwidth}
        \adjustbox{valign = c}{
            \includegraphics[width=\linewidth]{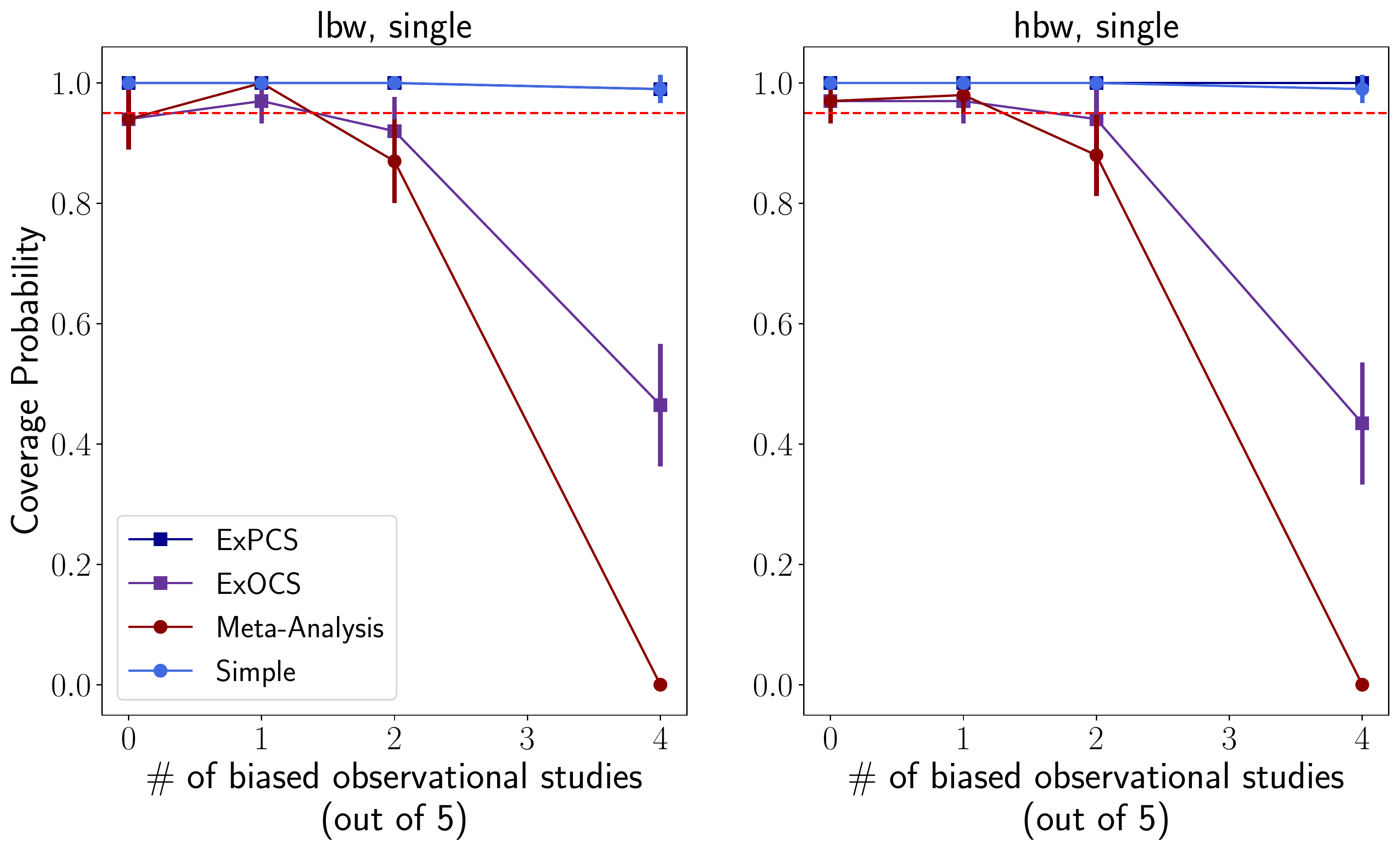}
        }
    \end{subfigure}
    \begin{subfigure}[t]{0.4\textwidth}
        \adjustbox{valign = c}{
            \centering
            \renewcommand\tabularxcolumn[1]{m{#1}}%
            \renewcommand\arraystretch{1.3}
            \setlength\tabcolsep{2pt}%
            \small 
            \begin{tabular}{cccccc}
                \toprule
                \multicolumn{6}{c}{Mean width of 95\% confidence intervals}\\
                \multicolumn{6}{c}{(4 biased studies)}\\
                \midrule
                && ExPCS & ExOCS & Meta & Simple \\
                \cmidrule{1-1} \cmidrule{3-6}
                    \textbf{LS} &&  5.29 &    3.55 &  2.34 & 5.51\\
                    \textbf{HS} &&  5.84 & 3.76 &  2.48 & 6.15 \\
                \bottomrule
            \end{tabular}
        }
    \end{subfigure}
    
    \caption{Coverage probabilities of confidence intervals shown as a function of the number of biased observational datasets (out of five). In the 4/5 biased studies case, the average interval widths for each approach is shown for two subgroups. We observe that \textit{ExPCS} achieves the best balance of interval width and coverage.}
    \label{fig:Biased}
\end{figure}

\begin{table}[h]
    \centering
    \begin{tabular}{ccccc}
        \toprule
        Upsampling ratio & 1.0 & 3.0 & 5.0 & 10.0\\
        \midrule 
            \thead{$P$(selecting \\ biased study)}    &    0.98          &  0.80 & 0.68    & 0.60       \\
        \bottomrule
    \end{tabular}
    \vspace{2mm}
    \caption{$P$(selecting biased study) as a function of upsampling ratio}
    \label{tab:select_bias}
\end{table}

\textbf{An analysis of including biased observational studies}: In \cref{fig:Biased}, we study coverage probability and width of confidence intervals in the presence of biased studies. \textit{Meta-Analysis} intervals approach zero coverage probability as the number of biased studies increases. Indeed, a fundamental assumption of this approach is that differences between estimates are only due to random variation, leading to poorer coverage probability when there are more biased studies. \textit{ExOCS} allows for elimination of biased studies in principle through falsification, resulting in improved coverage. However, it does not maintain the desired threshold of coverage ($95\%$), since biased estimators may still be included after falsification either due to chance or by being underpowered. Finally, \textit{ExPCS} and \textit{Simple Union} intervals have good coverage across the board, but as before, \textit{ExPCS} results in narrower intervals. 

Overall, we find that our method is robust to biased studies, yielding a good balance between coverage and width. In the case where one has adequate power, \textit{ExOCS} could be a reasonable alternative to get narrower intervals for a sacrifice in coverage (even in the presence of biased studies). However, this implicitly assumes that an estimator consistent for the validation effects will be consistent for the extrapolated effects. If this assumption does not hold, then \textit{ExOCS} will have poor coverage.

\textbf{Biased estimator selection}: In \cref{tab:select_bias}, we see that the probability of selecting the biased estimator goes down with increasing sample size of the observational studies, reflected by the increasing sample size ratio, $r$. This result validates our intuition that our method is more useful and results in more precise estimates of bias as we obtain more observational samples.

\section{Additional related work}\label{sec:app_related_work}
\textbf{Combining observational and experimental data} 
Prior work has sought to combine RCTs and observational studies for the purpose of more precise estimation of treatment effects~\citep{rosenman2020combining, rosenman2021propensity}, or for the purpose of generalizing or transporting estimates from RCTs to observational populations when overlap holds between the two (see \citet{Degtiar2021-bb} for a recent review). In contrast, our work is motivated by settings where there are populations in the observational studies who are not at all represented in trials, e.g., due to a lack of eligibility.  \citet{Kallus2018-ic} also seek to combine observational and experimental data to extrapolate beyond tfhe support of an RCT\@. They propose to learn a CATE function on observational data, and then learn a parametric additive correction term on the sample that overlaps between the RCT and observational data.  In contrast to this approach, we do not assume that confounding can be corrected for, and instead seek to choose an observational estimate (if one exists) that is already consistent for each sub-population.

\textbf{Calibrating observational confidence intervals} 
An alternative method for calibration of confidence intervals for observational studies makes use of negative controls \citep{Lipsitch2010-kg}, such as drug-outcome pairs known to have no causal relationship.  Methods range using  uses these negative controls to form an empirical null distribution for callibrated $p$-values\citep{Schuemie2014-je}, and \citet{Schuemie2018-gi} extend this approach to calibration of confidence intervals. These techniques have been used in several large-scale observational studies such as the LEGEND-HTN study for comparing antihypertensive drugs \citep{Suchard2019-cf}.  By contrast, our method does not assume the existence of negative controls, but instead uses a form of positive control (i.e., validation effects), one that is based on the same underlying treatment as the extrapolated effect.
\end{document}